\documentclass{article} 
\usepackage{iclr2021_conference,times}

\usepackage{amsmath,amsfonts,bm}

\def\Figref#1{Figure~\ref{#1}}

\def\Secref#1{Section~\ref{#1}}

\def\eqref#1{equation~\ref{#1}}
\def\Eqref#1{Equation~\ref{#1}}

\def\ceil#1{\lceil #1 \rceil}

\def\1{\bm{1}}

\DeclareMathAlphabet{\mathsfit}{\encodingdefault}{\sfdefault}{m}{sl}
\SetMathAlphabet{\mathsfit}{bold}{\encodingdefault}{\sfdefault}{bx}{n}

\usepackage{hyperref}
\usepackage{url}

\usepackage[utf8]{inputenc} 
\usepackage[T1]{fontenc}    
\usepackage{hyperref}       
\usepackage{booktabs}       
\usepackage{amsfonts}       
\usepackage{nicefrac}       

\usepackage[ruled,linesnumbered,vlined]{algorithm2e} 

\usepackage{etoolbox}
\makeatletter
\patchcmd{\@algocf@start}
  {-1.5em}
  {0pt}
  {}{}
\makeatother

\usepackage{amssymb, amsmath}
\usepackage[low-sup]{subdepth} 
\usepackage[mathscr]{euscript} 
\usepackage[bb=boondox]{mathalfa} 
\usepackage{nicefrac} 
\usepackage{setspace}  
\usepackage{ushort}  

\usepackage{changes}   
\definechangesauthor[name={Jaeyoung Lee}, color=orange]{JL}
\definechangesauthor[name={Sanghwa Lee}, color=blue]{SL}

\newcommand{\jaenote}[1]{\comment[id=JL]{#1}}
\newcommand{\sangnote}[1]{\comment[id=SL]{#1}}

\usepackage{picins}	
\usepackage{graphicx}   
\usepackage{wrapfig}  

\usepackage{subfigure}

\usepackage{tabularx,hhline}   
\usepackage{multirow} 
\usepackage{booktabs}
\usepackage{makecell}

\newcommand{\replaymemory}{\mathscr{D}}
\newcommand{\eDist}{\mathrm{Dist}_\replaymemory}
\newcommand{\indicator}{\mathbb{1}}
\newcommand{\tdinit}{TDInit }
\newcommand{\tdinitnos}{TDInit}
\newcommand{\tdclip}{TDClip }
\newcommand{\tdclipnos}{TDClip}
\newcommand{\tdpred}{TDPred }
\newcommand{\tdprednos}{TDPred}

\newcommand{\hdrule}{\midrule[\heavyrulewidth]}

\DeclareMathOperator*{\Argmax}{arg\,max}

\usepackage{amsthm}
\usepackage{xcolor}

\newtheorem{theorem}{Theorem}
\newtheorem{lemma}{Lemma}

\title{Predictive PER: {\Large Balancing Priority and Diversity towards Stable Deep Reinforcement Learning}}

\author{Sanghwa Lee$^1$, Jaeyoung Lee$^2$, Ichiro Hasuo$^1$\\
$^1$National Institute of Informatics, $^2$University of Waterloo \\
$^1$\{isoka, hasuo\}@nii.ac.jp, $^2$jaeyoung.lee@uwaterloo.ca
}

\iclrfinalcopy 
\begin{document}

\maketitle
\vspace{-10pt}
\begin{abstract}
Prioritized experience replay (PER) samples important transitions, rather than uniformly, to improve the performance of a deep reinforcement learning agent. We claim that such prioritization has to be balanced with sample diversity for making the DQN stabilized and preventing forgetting. Our proposed improvement over PER, called \emph{Predictive PER (PPER)}, takes three countermeasures (\emph{\tdinitnos, \tdclipnos, \tdprednos}) to (i) eliminate priority outliers and explosions and (ii) improve the sample diversity and distributions, weighted by priorities, both leading to stabilizing the DQN. The most notable among the three is the introduction of the second DNN called \emph{\tdprednos} to generalize the in-distribution priorities. Ablation study and full experiments with Atari games show that each countermeasure by its own way and PPER contribute to successfully enhancing stability and thus performance over PER.
\end{abstract}

\vspace{-5pt}
\section{Introduction}

Deep reinforcement learning (DRL) has come to be a promising methodology for solving sequential decision-making problems. Its performance surpassed the human level in the game of Go \citep{Go2017}, and many of the Atari games \citep{Atari2015}. This success technically relies on the advances in the deep learning plus a series of methods---slow target network update, experience replay, etc. (e.g., see \citealp{Atari2013, Atari2015})---to alleviate the detrimental effects of the non-stationary nature, temporal data correlation, and the inefficient use of data.

\textbf{Prioritized Experience Replay.}
Experience replay (ER) stores sequential transition data, called \emph{experiences} or \emph{transitions}, into the memory and then sample them uniformly to (re-)use in the update rules \citep{ER1992}. \citet{Atari2013, Atari2015} successfully implemented this technique with deep Q-network (DQN) to un-correlate the experiences and improve both sample efficiency and stability of the training process. Furthermore, \citet{PER2015} proposed prioritized experience replay (PER), which inherits the same idea but samples the transitions in the memory according to the distribution determined by the so-called priorities, rather than uniformly. In PER, a priority is assigned to each experience of the replay memory. It is either proportional to its absolute TD-error (proportional-based) or inversely to the rank of the absolute TD-error in the memory (rank-based). Such extensions of ER were experimentally validated on the Atari games in the framework of deep Q-learning (DQL); the PER showed the promising performance improvement over the uniform ER \citep{DuelingDQN2015,PER2015,Rainbow2018}. 



\textbf{Predictive PER.}
In this paper, we improve PER \citep[proportional-based]{PER2015}. 
We claim that too much biased prioritization in PER, due to outliers, spikes, and explosions of priorities, harm the sample diversity, leading to destabilizing the DQN and forgetting. While the uniform ER has the maximum sample diversity, we take the following three countermeasures to achieve a good balance between priority and diversity. The first one is \emph{\tdinitnos}: for new experiences, PER assigns
 the maximum priority ever computed; we instead assign priorities proportionally to their TD errors, as batch priority updates work in the original PER. The second is \emph{\tdclipnos}, which upper- and lower-clips priorities using stochastically adaptive thresholds.
 Finally and most notably, we use a DNN called \emph{\tdprednos} that is trained to estimate TD errors (hence, priorities). The generalization capabilities of the \tdpred can aggressively stabilize the priority distribution hence the training procedure. These three ideas compose our proposed method, called  \emph{predictive PER (PPER)}. It utilizes two DNNs (DQN and \tdprednos); we mitigate the additional training cost by letting \tdpred reuse DQN's convolution layers. \Figref{fig:overview} illustrates an overview of the proposed PPER; related works are given in Appendix~\ref{appendix:related works}.





\begin{figure}[tbp!]
\includegraphics[width=0.95\columnwidth, height=0.35\columnwidth]{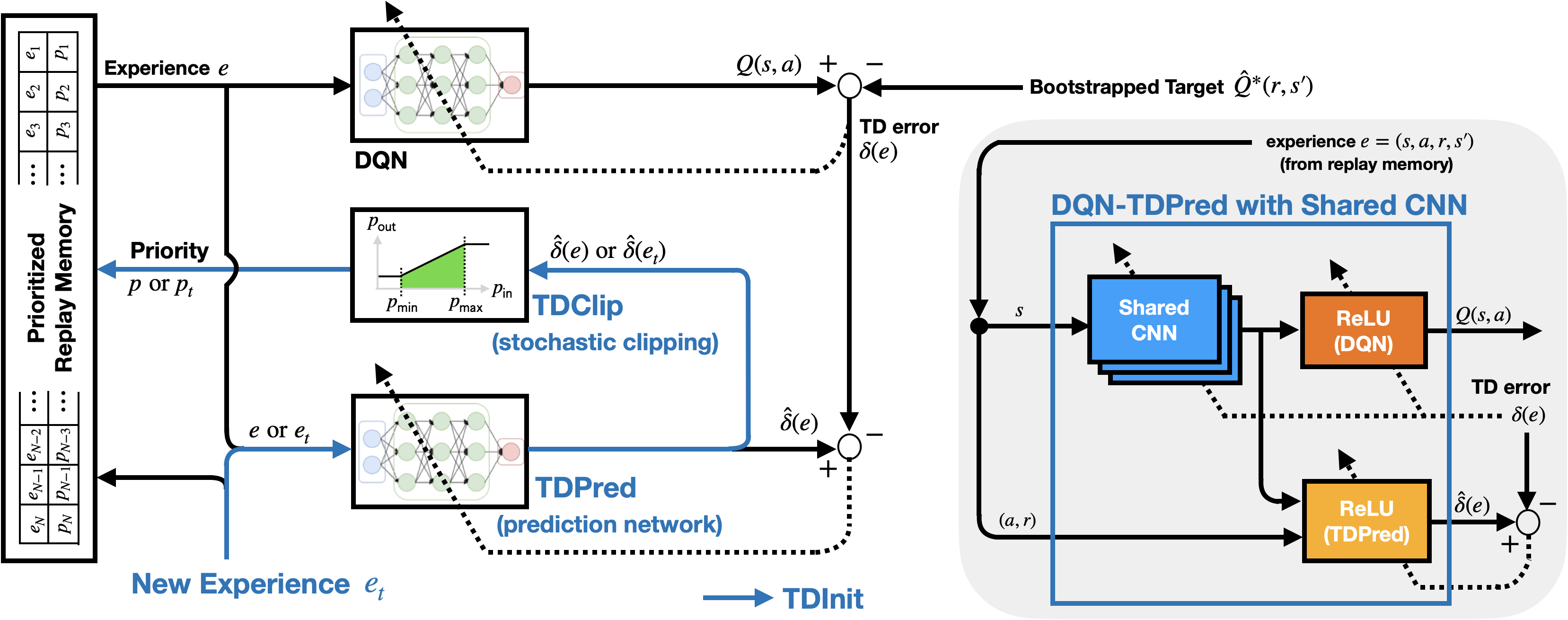}
\vspace{-1.0em}
\caption{Overview of the proposed method, predictive PER (PPER), where the (initial) priorities are determined based on the TD-error (TDInit), its statistical clipping (TDClip), and its prediction (TDPred). TDPred reuses DQN's convolution layers that are trained by DQL (right).}	
\label{fig:overview}
\vspace{-1.5em}
\end{figure}

\paragraph{Atari Experiments} We experimentally validate the stability and performance of the three countermeasures and PPER by comparing them with PER, all applied to Atari games with double DQN and dueling network structure. We consider 5 and total 58 Atari tasks for the ablation study in \Secref{section:PPER} and the full evaluation in \Secref{sec:Atari experiments}, respectively. In the experiments, $50$ millions ($=T_\mathsf{max}$) training steps are performed with consistent hyper-parameters (see Appendix~\ref{appendix:experimental settings} for details).

\vspace{-0.5em}

\setlength{\intextsep}{0pt} 

\section{Prioritized Experience Replay, Revisited}

\vspace{-0.5em}

To motivate our work, we revisit the original PER \citep{PER2015}. It has a \emph{prioritized replay memory} $\replaymemory$. Each data $e \in \replaymemory$ has a priority $p \geq 0$; when a new experience comes and its length $N := |\replaymemory|$ reaches the maximum capacity, the oldest is removed from $\replaymemory$. Algorithm~\ref{alg:PER} describes the DQL with PER.
At each time $t = 1, \cdots, T_\mathsf{max}$, the agent stores the observed transition tuple $e_t = \langle s_{t-1}, a_{t-1}, r_t, s_t \rangle$ into $\replaymemory$, with its priority $p_t = p_\mathsf{max}$ (lines~\ref{line:PERApplyAction}--\ref{line:PERStoreExpr}), where $p_\mathsf{max}$ denotes the largest priority ever computed (line~\ref{line:PERPriorUpdate}); the action $a_{t-1}$ is sampled and applied to the environment, according to the distribution $\pi_\theta(s_{t-1})$ on the finite action space $\mathscr{A}$. The behavior policy $\pi_\theta$ maps each state $s$ to a distribution on $\mathscr{A}$ and typically depends on the DQN $Q_\theta$ (e.g., $\epsilon$-greedy). 

Every $t_\mathsf{replay}$ instant, a batch of a $K$-number of transitions $e = \langle s, a, r, s' \rangle$'s are sampled from $\replaymemory$, from the \emph{experience distribution} $\eDist$ on $\replaymemory$ defined by the priorities as
\begin{equation}
 	\eDist(i) \,:=\, p_i^\alpha \,\big /\, {\smash{\textstyle \sum_{j = 1}^N} p_j^\alpha} \qquad\quad (\alpha \in [0, 1])
 	\label{eq:probPrioritySampling}
\end{equation}
at each index $i = 1, \cdots, N$. The selected batch is then used to calculate the gradient $\nabla_\theta Q_\theta(s, a)$ and the TD error $\delta(e) = {\hat Q}^*(s', r; \theta, \theta^-) - Q_{\theta}(s, a)$, by which the agent updates all the batch priorities~$p$'s, the largest priority $p_\mathsf{max}$ ever computed, and with importance sampling (IS), the DQN weights $\theta$ (lines~\ref{line:PERBatchUpdateStart}--\ref{line:PERUpdateTheta}). The priority~$p$ is updated proportionally to $|\delta|$~(line~\ref{line:PERPriorUpdate}). The bootstrapped target $\hat Q^*(r, s'; \theta, \theta^-) := r + \gamma \cdot Q_{\theta^-}(s', \textstyle \Argmax_{a' \in \mathscr{A}} Q_\theta(s', a'))$ \citep{DDQN2016}, where $Q_{\theta^-}$ is the target DQN that has the same structure as $Q_\theta$, with its weights $\theta^-$. Every $t_\mathsf{target}$ instant, the target DQN $Q_{\theta^-}$ copies the DQN weights $\theta$, i.e., $\theta^- \gets \theta$ (line~\ref{line:PERUpdateThetaMinus}). This happens with a large enough $t_\mathsf{target}$, e.g., $t_\mathsf{target} = \smash{10^4}$ and $t_\mathsf{replay} = 4$ in our experiments and \cite{PER2015}'s.

\begin{figure*}[t!]
\subfigure[Seaquest, a good example]{\includegraphics[width=\columnwidth]{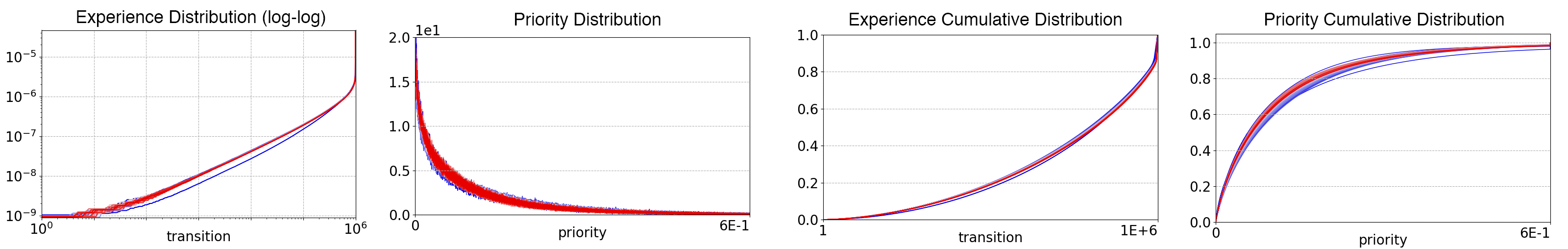}}
\subfigure[Zaxxon, a bad example]{\includegraphics[width=\columnwidth]{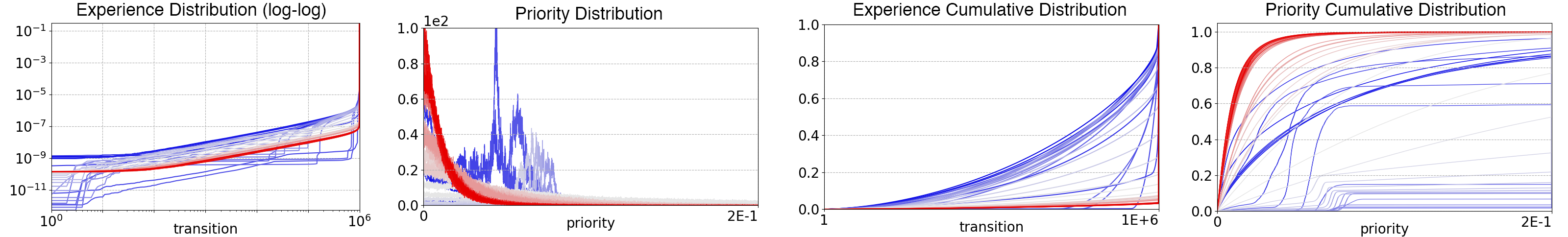}}
\vspace{-1.5em}
\caption{(Cumulative) distributions in PER, from {\color{blue} blue} ($t = 1$) to {\color{red} red} $(t = T_\mathsf{max}$).}
\vspace{-1.5em}
\label{fig:intro:distributions}
\end{figure*}

For notational convenience, we assume without loss of generality that the priorities are sorted, so that $p_i \leq p_j$ and thus $\eDist(i) \leq \eDist(j)$ if $i \leq j$. Similarly, $p_{i_k} \! \leq p_{i_l}$ for any batch $\langle p_{i_k} \rangle_{k = 1}^K$ if $k \leq l$. With this convention, the last priorities $p_N$ and $p_{i_K}$ denote the maximum priorities $\max_{1 \leq i \leq N} p_i$ in the memory $\mathscr{D}$ and $\max_{1 \leq k \leq K} p_{i_k}$ in the batch, respectively. We also define \emph{priority distribution} $\mathscr{P}_\mathscr{D}(p) := \smash{\sum_{i = 1}^N} \mathbb{1}(p = p_i) \slash N$, where $\indicator(\cdot)$ is the indicator function. Both distributions $\eDist$ and $\mathscr{P}_\mathscr{D}$ are the main topic of this work. Figure~\ref{fig:intro:distributions} shows the (a) good and (b) bad examples of those distributions (\textbf{left}) in PER and their cumulative ones (\textbf{right}).

Experience sampling (line~\ref{line:PERSampleTrans}) is central to PER. It provides a trade-off between \emph{priority} and \emph{diversity} of samples. At one extreme, 
if $\eDist = 1 / N$ (i.e., $\alpha = 0$ or $p_i$'s are all the same for all $i$), then PER is of no effect and reduced to uniform sampling (maximum diversity, minimum prioritization). On the other hand, if 
$
	\eDist (i) = \indicator(p_i = p_N) \slash N'
$
for $N' := N \! \cdot \! \mathscr{P_D}(p_N)$, then PER always chooses the experiences that have the maximum priority $p_N$ (minimum diversity). \Eqref{eq:probPrioritySampling} lies in between these two extremes, balancing priority and diversity of sampled data. Here, we note that an adequate level of diversity in the samples can be a key to stable DRL --- if the diversity is too low, then the agent would use just a tiny highly-prioritized subset of $\mathscr{D}$ and never sample any of the other experiences in $\mathscr{D}$ for a long time, starting to (slowly) forget what it has already learned! Figure~\ref{fig:intro:distributions}(b) (e.g., see \textbf{middle-right}) shows such an example, where transitions to the impulsive experience distribution results in disastrous loss of diversity in samples, leading to severe forgetting, eventually. In short, \emph{PER needs to maintain a certain level of diversity of experience samples to prevent such forgetting.} The proposed PPER and its three countermeasures (TDInit, TDClip, TDPred) are designed for such purpose.

\vspace{0.3em}

\begin{algorithm}[h]
	\label{alg:PER}
    \caption{Prioritized Experience Replay \citep{PER2015}}
    \begin{spacing} {1}
    \textbf{Initialize:} the DQN weights $\theta$; the target DQN weights $\theta^- \gets \theta$\;
    Observe an initial state $s_{0}$; set the replay memory $\mathscr{D} = \varnothing$ and $p_\mathsf{max} \gets 1$\;
    \For{$t = 1 \; \mathrm{ to } \; T_\mathsf{max}$}
    {
        Apply an action $a_{t - 1} \sim \pi_\theta(s_{t-1})$ and observe the reward $r_{t}$ and the next state $s_{t}$\;  \label{line:PERApplyAction}
        Store the experience $e_t := \langle s_{t-1}, a_{t-1}, r_t, s_t \rangle$ into $\replaymemory$, with its priority $p_t \gets p_\mathsf{max}$\; \label{line:PERStoreExpr}
        \If($\Delta \gets 0$;){$t \emph{ mod } t_\mathsf{replay} = 0$ \label{line:PERBatchUpdateStart}} 
        {
        	\For{$k = 1 \; \mathrm{ to } \; K$}
        	{
        		Sample a transition index $i_k \sim \eDist(\cdot)$ of $\replaymemory$ (\Eqref{eq:probPrioritySampling})\; \label{line:PERSampleTrans}
	            Compute the IS weight $w_{k} \gets (N \cdot \eDist(i_k))^{-1}$\;  \label{line:PERISWeight}
	            Get the experience $e := \langle s, a, r, s' \rangle \in \replaymemory$ at the index $i_k$\; \label{line:PERGetExpr}
	            \vspace{-1.5pt}
		        Compute the TD error $\delta \gets {\hat Q}^*(r, s'; \theta, \theta^-) - Q_\theta(s, a)$\; \label{line:PERTDError}
	            \vspace{0.5pt}		        
		        Update the priority $p \gets |\delta|$ at the index $i_k$ of $\replaymemory$; then, set $p_\mathsf{max} \gets \max\{p_\mathsf{max}, p\}$\; \label{line:PERPriorUpdate}
	            Accumulate the weight update: $\Delta \gets \Delta + \smash{w_k^\beta} \cdot \delta \cdot \nabla_\theta Q_\theta(s, a)$\; \label{line:PERAccWeightUpdate}
	        }
	        \vspace{-2.5pt}
	        Update the weights $\theta \gets \theta + \eta \cdot \Delta \big / \big ( \max_{k} w_k^\beta \big )$\; \label{line:PERUpdateTheta}
        }
	    \vspace{-1.5pt}
        $\theta^- \gets \theta$ \textbf{if} $t \text{ mod } t_\mathsf{target} = 0$\; \label{line:PERUpdateThetaMinus}
    }
    \vspace{-2.5pt}
    \end{spacing}
\end{algorithm}

\vspace{-0.5em}

\subsection{Priority Outliers and Explosions}
\label{subsec:priority outliers}

\vspace{-0.5em}

\emph{Priority outliers} are experiences with extremely large or small priorities: they make the experience distribution $\eDist (i)$ towards $\indicator(p_i = p_N) \slash N'$, losing diversity in samples hence stability of the DQN, as discussed above. \emph{Priority explosions} refer to abrupt increases of priorities over time: they generate a bulky amount of extremely large priorities in a period of time, making the current priorities relatively but extremely small, hence resulting in a massive amount of priority outliers, with unstable priority and experience distributions. Once the explosion happens, the DQN weights $\theta$ can rapidly change since $\eDist$ and thus the DQN inputs do. If the behavior policy~$\pi_\theta$ becomes significantly different or bad accordingly, then there would become no way to recover the past experiences, meaning that the DQN will be forgetting what it has already learned, on and on. Figures~\ref{fig:intro:distributions}(b) (\textbf{right-most}, {\color{blue} blue curves}) and \ref{fig:ablation:various figs}(d) (high variance TD error) illustrate such outliers in PER and when the explosion happens, respectively.

The main working hypothesis of the current paper is that, for the above reasons, \emph{priority outliers and explosions should be avoided}. We shall now discuss how they occur in PER. 





\begin{wrapfigure}[13]{r}{0.4\linewidth}
	\includegraphics[width=0.4\columnwidth]{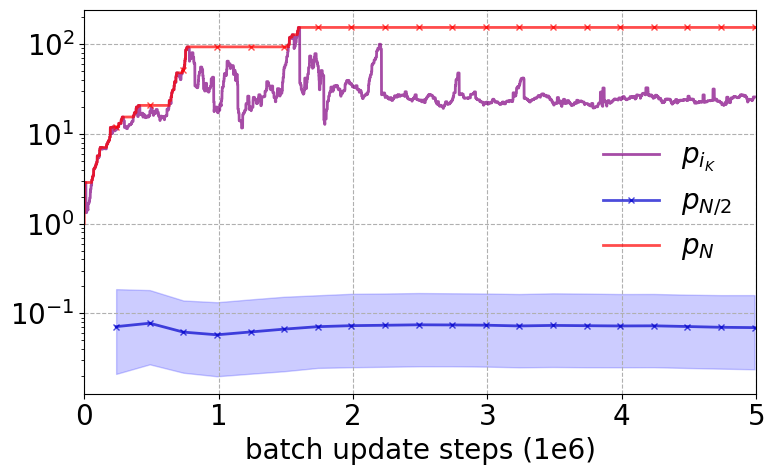} 	
	\vspace{-1.5em}
	\caption{Smoothed trajectories of priorities $p_{i_K}$, $p_{\ceil{N/2}}$, and 
	$p_N$ (Seaquest).
	}
	\vspace{-1.5em}
	\label{fig:pmax}
\end{wrapfigure}

\paragraph{Choice of Initial Priorities} Our first observation is that the choice of an initial priority in PER causes priority outliers. According to line~\ref{line:PERStoreExpr},  the initial priority $p_t$ is the largest priority $p_\mathsf{max}$ ever computed in the PER that is in general an upper bound of the priority maximum $p_N$. In PER, $p_N = p_\mathsf{max}$ always since the initial priorities $p_t$'s set to $p_\mathsf{max}$  come into play (line~\ref{line:PERStoreExpr}). However, the TD error $|\delta|$, and thus the priorities $p$'s once updated by $p \gets |\delta|$ (line~\ref{line:PERPriorUpdate}), eventually become much smaller, as training progresses and comes to stabilization.
Moreover, by the update $p_\mathsf{max} \gets \max\{p_\mathsf{max}, p\}$ (line~\ref{line:PERPriorUpdate}), $p_\mathsf{max}$ can ever increase. As a consequence, $p_\mathsf{max}$ and thus all initial priorities $p_t$'s can become extremely large and out of the current priority distribution (e.g., see Figures~\ref{fig:intro:distributions}(b)), meaning that they act as \emph{priority outliers}. The experimental result in \Figref{fig:pmax} supports this claim: it shows the trajectory of $p_\mathsf{max}$ ($= p_N$) and the transition of the median priority \smash{$p_{\lceil N/2 \rceil}$}. We see the large gap between the two, and that $p_\mathsf{max}$ never decreases even though the maximal priority $p_{i_K}$ in the batch (violet) starts to decrease at early steps. \tdinit provides a solution for this issue due to the ever increasing $p_\mathsf{max}$ (lines~\ref{line:PERStoreExpr} and \ref{line:PERPriorUpdate}).

\paragraph{Distributions and Spikings of TD-errors}
TD errors, thus priorities, have a distribution that is dense near zero, as shown in Figure~\ref{fig:intro:distributions}(a) and (b) (see \textbf{middle-left}). Such a distribution can induce lots of experiences with \emph{almost zero priorities} that are useless as they are almost never sampled \citep{ExperienceSelection2018}. The other issue is that the TD error 
$\delta$ can \emph{spike}. Spikes in $\delta$ affect both the initial priorities, by making $p_\mathsf{max}$ larger, and the batch priority update (lines~\ref{line:PERStoreExpr} and \ref{line:PERPriorUpdate}), both causing priorities to spike. Such spikes typically occur with (i) new experiences that have not explored or well-exploited yet and (ii) old ones that were not replayed for a long time; then, the TD errors $|\delta|$'s tend to be initially very high or different from the previous', respectively. \Figref{fig:pmax} shows that this is indeed the case: for example, the maximal batch priority $p_{i_K}$ are seen to be far above the median priority $p_{\ceil{N/2}}$ in $\mathscr{D}$. A series of TD-error spikes dense-in-time can rapidly ruin the sample diversity thus induce a \emph{priority explosion}. 
This is especially when the agent learns a new task or stays in a region that is not explored before. In this case, most of the new TD-errors are expected to be large relative to the old ones, hence may act as spikes, eventually losing lots of past data without sampling them enough. We use two ideas to alleviate this difficulty: \tdclip and \tdprednos.

\section{Predictive Prioritized Experience Replay}
\label{section:PPER}

We made a hypothesis: \emph{priority outliers and explosions harm the stability of PER}. We now introduce the three countermeasures for reducing them: TD-error-based initial priorities (\tdinitnos), statistical TD-error clipping (\tdclipnos), and TD-error prediction network (\tdprednos). Their combination constitutes our proposed  \emph{Predictive PER (PPER)}. In this section, PPER and each idea are experimentally evaluated via ablation study on 5 Atari games as shown in Figures~\ref{fig:ablation study:training curves} and \ref{fig:ablation study:summary} (see Appendix~\ref{appendix:ablation study} for more).

\vspace{-0.5em}

\subsection{TDInit: TD-error-based Initial Priorities}

\vspace{-0.5em}

A possible way of avoiding the outliers---resulting from $p \gets p_\mathsf{max}$---is to take the current median priority $p_{\lceil N/2 \rceil}$, replacing line~\ref{line:PERStoreExpr} with $p_{t}\gets p_{\lceil N/2 \rceil}$. \tdinit is more informative: it takes $ |\delta_t|$ as the initial priority. Therefore, line~\ref{line:PERStoreExpr} becomes $p_{t}\gets |\delta_t|$ that is now the same as the batch priority update rule (line~\ref{line:PERPriorUpdate}). With this modification, the TD errors $|\delta|$'s and thus the maximal priority $p_N$ are expected to decrease as the training goes on and on, whereas $p_N$ ($= p_\mathsf{max}$) in PER never decreases.

As shown in Figure~\ref{fig:ablation study:training curves}, TDInit can resolve severe forgetting of PER (BattleZone; Robotank) or even improve overall performance (Robotank; Seaquest). However, TDInit alone may not be enough since other types of priority outliers can occur, which cause priority explosions (e.g., TD-error spikes). Figure~\ref{fig:ablation study:training curves} shows such examples (Tennis;  Zaxxon), where \tdinit cannot prevent the severe forgetting or even make it worse. Also notice a lot of distributional variations with priority outliers and explosions (Figures~\ref{fig:ablation:various figs}(a) and (d), resp.). Bounding the priorities (e.g., TDClip) or reducing their (distributional) variance (e.g., TDPred) can solve or significantly alleviate such phenomena.

\begin{figure*}[tbp!]
\includegraphics[width=\columnwidth]{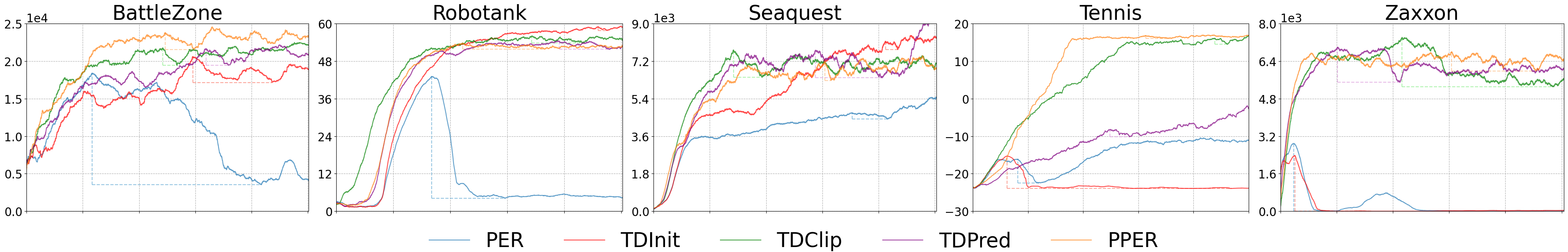}
\vspace{-2.0em}
\caption{Training scores of PER, PPER, and the three countermeasures, up to $T_\mathsf{max} = 50$M steps.}	
\vspace{-1.0em}
\label{fig:ablation study:training curves}
\end{figure*}

\begin{figure*}[t!]
\subfigure[\tdinitnos, Tennis]{\includegraphics[width=0.245\columnwidth]{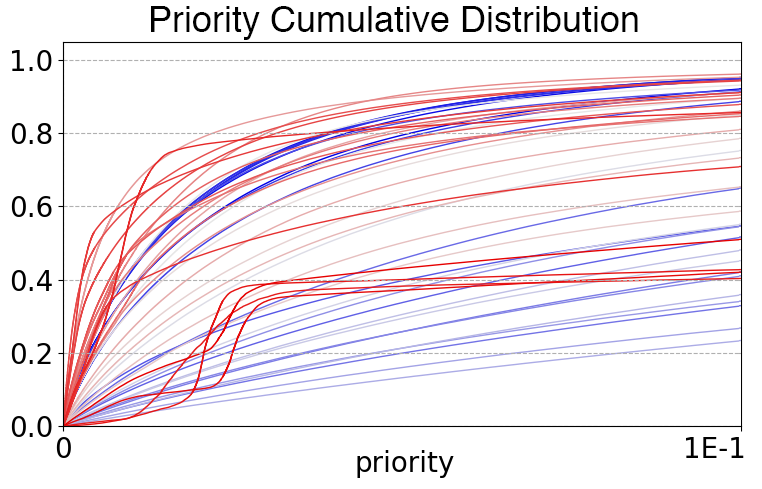}}
\subfigure[\tdclipnos, Tennis]{\includegraphics[width=0.245\columnwidth]{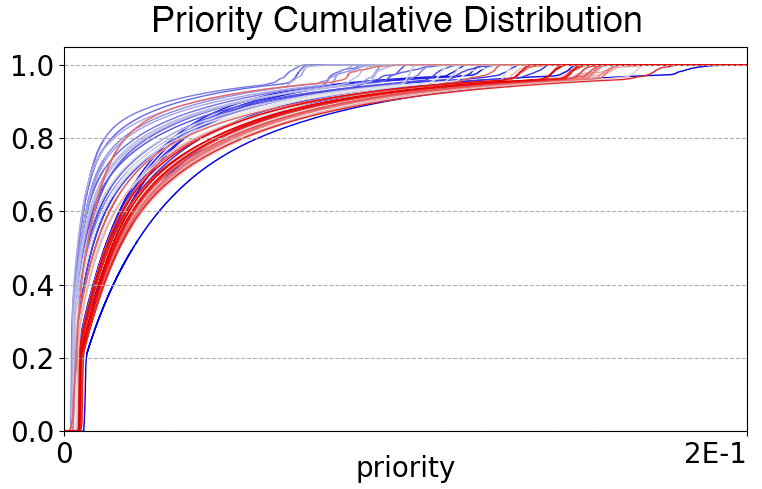}}
\subfigure[\tdpred, Zaxxon]{\includegraphics[width=0.245\columnwidth]{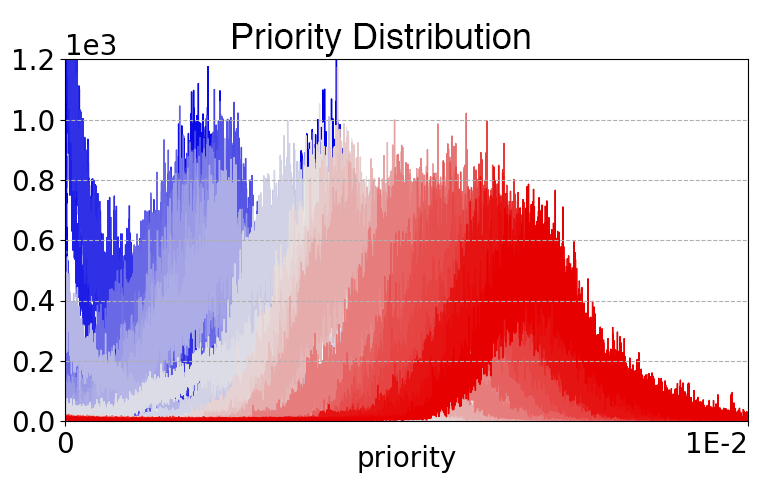}}
\subfigure[Zaxxon]{\includegraphics[width=0.245\columnwidth]{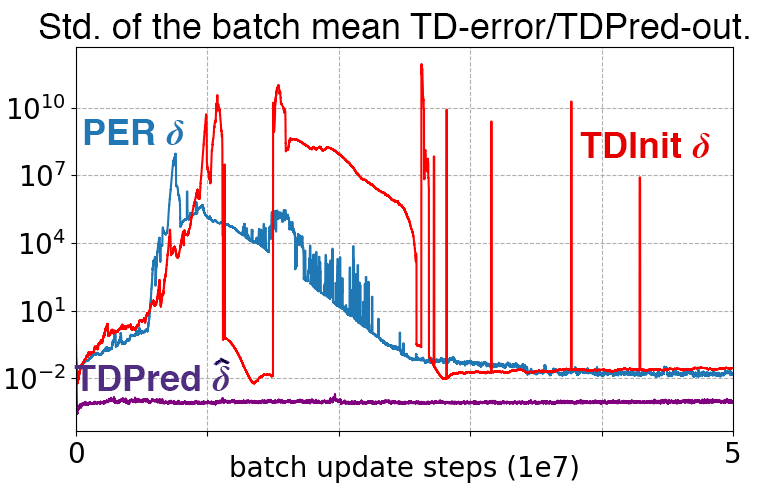}}
\vspace{-1.5em}
\caption{(a)--(c) Priority distributions; (d) std. of batch mean of $\delta$ (PER/\tdinitnos) and $\hat \delta$ (\tdprednos)}
\vspace{-1.0em}
\label{fig:ablation:various figs}
\end{figure*}

\subsection{TDClip: Statistical TD-error Clipping}
\label{subsection:statistical priority clipping}

\vspace{-0.5em}

\tdclip clips the  absolute TD errors to remove priority outliers.
One naive approach of such is by fixed clipping bounds. However, such a scheme can be ineffective since the TD errors can be large in the beginning but tend to decrease in magnitudes. Their statistics change over the training. 

To come up with such difficulty, we actively estimate the average $\mu := \sum_{i=1}^N |\delta(e_i)| / N$ of the current TD errors $|\delta(e)|$'s of all the experiences $e$ in $\mathscr{D}$. In each priority update, we use the estimate $\tilde p$ of $\mu$ to statistically clip the priority in its update rule: $p \gets \mathrm{median} \smash{\big \{ p_\mathsf{min}, \, |\delta|, \, p_\mathsf{max} \big \}}$.
Here, the clipping bounds $\langle p_\mathsf{min}, p_\mathsf{max} \rangle$ are initialized to $\langle 0, 1\rangle$ and updated to $\langle \rho_\mathsf{min}, \rho_\mathsf{max} \rangle \cdot \tilde p$ whenever $\tilde p$ is updated. Our design choice is $\langle \rho_\mathsf{min}, \rho_\mathsf{max} \rangle = \langle 0.12, 3.7 \rangle$, which makes $p_\mathsf{min}$ and $p_\mathsf{max}$ the $0.1\sigma$ and $3\sigma$ points of the distribution of TD error $\delta$ over $\mathscr{D}$ if $\delta \sim \mathcal{N}(0, \sigma)$ (Appendix~\ref{appendix:details on SPC}). 
At each batch update step $k$, we update $\tilde p$ with an estimate $\Delta_\mathsf{a} := \sum_{k=1}^K w_k \cdot |\delta(e_{i_k})| / K$ of the current $\mu$ via
\begin{align}
		\kappa \gets \lambda \cdot \kappa + 1
\quad\text{and then}\quad
		\tilde p \gets \tilde p + \big (\Delta_\mathsf{a} - \tilde p \big ) / \kappa
	\label{eq:AdaClip update rule}
\end{align}
where both $\kappa$ and $\tilde p$ are initialized to zero; the forgetting factor $\lambda \in (0, 1)$, which we set $\lambda = 0.9985$ in all of the experiments, compensates the non-stationary nature of the TD error distribution. For more details with theoretical analysis, see Appendix~\ref{appendix:details on SPC}. 

Figure~\ref{fig:ablation study:training curves} shows that TDClip prevents forgetting and thus achieves higher scores than PER's, in all 5 Atari games including those where TDInit cannot (Tennis; Zaxxon)---TDClip effectively clips the outliers that exist in PER and TDInit as shown in Figure~\ref{fig:ablation:various figs}(b). Also note that by TDClip, $p_\mathsf{max}$ can now decrease and be located at the edge of the priority in-distribution, without \tdinitnos.

\vspace{-0.5em}

\subsection{TDPred: TD-error Prediction Network}
\label{subsection:PPN}

\vspace{-0.5em}


The TD-error prediction network (TDPred) \smash{${\hat \delta} \equiv \hat \delta_\vartheta$}, parameterized by $\vartheta$, smoothly generalizes the in-distribution experiences in $\mathscr{D}$ with a variance-reduced Gaussian-like distribution, by minimizing the mean squared loss 
$
L_\mathscr{D}(\vartheta) := \mathbb{E}_{i \sim \eDist(\cdot)} \big [ \, (\hat \delta_\vartheta(e_i) - \delta(e_i))^2 \, \big ]
$
in an online supervised manner, where $\delta(e)$ and $\smash{\hat \delta}_\vartheta(e)$ denote the TD error and the TDPred output, w.r.t. the input experience $e \in \mathscr{D}$. 
The corresponding stochastic gradient descent update is:
$
		\vartheta \gets \vartheta + \eta_\mathsf{p} \cdot \big ( \delta(e_i) - \smash{\hat \delta_\vartheta(e_i)} \big ) \cdot \nabla_\vartheta {\hat  \delta}_\vartheta(e_i)
$
for a learning rate $\eta_\mathsf{p}$, where the index $i \sim \eDist(\cdot)$. The idea of TDPred is to employ its output to assign the priority as $p \gets | \smash{\hat \delta_\vartheta}|$, and its effect on $p$ has three folds below. In short, \tdprednos's generalization capabilities allow us to smooth out TD errors---especially for new experiences on which the DQN is not trained enough---and avoid priority outliers, spikes, and explosions. Figure~\ref{fig:ablation study:training curves} shows that the stability is indeed enhanced as TDPred shows no severe forgetting while PER does.

\vspace{-0.5em}

\paragraph{Regularization} The \tdpred output $\smash{\hat \delta}_\vartheta$ varies over time more slowly than the TD error $\delta$, due to its $\vartheta$-update process. Hence, the use of TDPred regularizes the abrupt distributional changes in the priorities to avoid them from being outliers and prevent their explosions (see Figure~\ref{fig:ablation:various figs}(c) and (d)).

\vspace{-0.5em}

\paragraph{Gaussian-like Priority Distribution} The output of a wide DNN is typically Gaussian \citep{Gaussian2018}. Likewise, as in Figure~\ref{fig:ablation:various figs}(c) and Appendix~\ref{appendix:ablation study}, the priority~$p$ of TDPred (under small variance), updated proportionally to $|\hat \delta_\vartheta|$, has a Gaussian-like distribution, eventually --- even for the case where the network is not wide, except a few (Appendices~\ref{appendix:ablation study} and \ref{appendix:subsection:PPN with smaller layer}). This can reduce priority outliers as the distribution is symmetric and heavily centered at the mean, whereas the power-law priority distributions in PER and TDInit are heavily tailed, skewed towards and dense near zero.


\vspace{-0.5em}

\paragraph{Variance Reduction}	A modern analysis shows that DNNs can significantly reduce both variance of the outputs and bias from the target \citep{Neal2018}, overcoming traditional bias-variance dilemma \citep{Geman1992}. The use of TDPred therefore results in variance reduction of its outputs $\smash{\hat \delta}_\vartheta$'s (hence, the priorities $p$'s) while keeping the bias reasonably small (i.e., $\big |\mathbb{E}_{i \sim \eDist(\cdot)}[ \delta(e_i) - \hat \delta_\vartheta(e_i) ] \big | \ll 1$). Figure~\ref{fig:ablation:various figs}(d) and Appendix~\ref{appendix:ablation study} show that this is indeed true for all 5 Atari games --- for the sampled experience $e_{i_k} \in \mathscr{D}$, the variances of $|\smash{\hat \delta}_\vartheta|$'s and $p$'s in TDPred are typically at least 10 times smaller than those of $|\delta|$ in PER, while keeping the error $|\smash{\hat \delta}_\vartheta| - |\delta|$ reasonably small. As a result, TDPred removes \emph{high-variance priority outliers} and makes the distribution more centered and denser around the mean, \emph{making itself insensitive to large TD-error deviations hence explosions}. Regarding the bias, we note that as in rank-based prioritization \citep{PER2015}, the high precision of the TD errors is not always necessary. 

\textbf{Our Design Choices} We empirically found that TDPred must be a ConvNet to learn $\delta$. We let TDPred share DQN's convolutional layers, which has two advantages: 1) establishing the common ground of two function approximators, 2) reducing computational cost of TDPred's update. Update of TDPred is only relevant to its fully-connected layers (see \Secref{sec:Atari experiments} and Appendix \ref{appendix:experimental settings} for more).

\vspace{0.5em}

\begin{algorithm}[h!] 
	\label{alg:proposed algorithm}
    \caption{Predictive Prioritized Experience Replay (PPER)}
    \begin{spacing} {0.9}
    \textbf{Initialize:} 
    the DQN weights $\theta$; the target network weights $\theta^- \gets \theta$; the TDPred weights $\vartheta$\;
    Observe $s_{0}$; set the replay memory $\mathscr{D} = \varnothing$; $\kappa \gets 0$ and $\langle p_\mathsf{min}, \;\tilde p, \; p_\mathsf{max} \rangle \gets \langle 0, 0, 1 \rangle$\;    
    \For{$t = 1 \; \mathrm{ to } \; T_\mathsf{max}$}
    {
        Apply an action $a_{t - 1} \sim \pi_\theta(s_{t-1})$ and observe the reward $r_{t}$ and the next state $s_{t}$\; \label{line:PPERApplyAction}
        Compute TDPred output ${\hat \delta}_t \gets {\hat \delta}_\vartheta(e_t)$ for the transition $e_t := \langle s_{t-1}, a_{t-1}, r_t, s_{t} \rangle$\; \label{line:PPERPPNOutput}
        Compute priority $p_t \gets \mathrm{median} \big \{ p_\mathsf{min}, \,|\smash{\hat \delta}_t|, \, p_\mathsf{max}  \big \}$ and store $(e_t, p_t)$ in $\mathscr{D}$\; \label{line:PPERStoreExpr}
        \If($\langle \Delta_\mathsf{q}, \Delta_\mathsf{p}, \Delta_\mathsf{a} \rangle \gets \langle 0, 0, 0 \rangle$;){$t \emph{ mod } t_\mathsf{replay} = 0$}
        {
			\label{line:PPERInitBatchUpdate}
        	\For{$k = 1 \; \mathrm{ to } \; K$ \label{line:PPERBatchUpdateStart}}
        	{
        		Sample a transition index $i_k \sim \eDist(\cdot)$ and compute $w_{k} \gets (N \cdot \eDist(i_k))^{-1}$\; \label{line:PPERSampleTrans}
	            Get the experience $e := \langle s, a, r, s' \rangle \in \mathscr{D}$ at the index $i_k$\; \label{line:PPERGetExpr}
		        Compute 
		        TD error $\delta \gets {\hat Q}_*(r, s'; \theta, \theta^-) - Q_\theta(s, a)$ and 
		        TDPred output ${\hat \delta} \gets {\hat \delta}_\vartheta(e)$\;
		        \label{line:PPERTDError}
		        Update the priority $p \gets \mathrm{median} \big \{ p_\mathsf{min}, \,|\smash{\hat \delta}|, \, p_\mathsf{max}  \big \}$ at the index $i_k$ of $\mathscr{D}$\; \label{line:PPERPriorUpdate}
	            $
					\big \langle \Delta_\mathsf{a}, \Delta_\mathsf{p}, \Delta_\mathsf{q} \big \rangle 
					\gets 
					\big \langle \Delta_\mathsf{a}, \Delta_\mathsf{p}, \Delta_\mathsf{q} \big \rangle 
					\, + \, \big \langle 
						w_k \cdot |\delta| / K, \;\; 
						(\delta - \hat \delta) \cdot \nabla_\vartheta {\hat \delta}_\vartheta(e), \;\; 
						w_k^\beta \cdot \delta \cdot \nabla_\theta Q_\theta(s, a) 
					   \big \rangle
	            $\;
	            \label{line:PPERAccUpdate}
	        }
	        \vspace{-2.5pt}
	        Update the weights $\theta \gets \theta + \eta_\mathsf{q} \cdot \Delta_\mathsf{q} \big / \big ( \max_{k} w_k^\beta \big )$ and $\vartheta \gets \vartheta + \eta_\mathsf{p} \cdot \Delta_\mathsf{p}$\; \label{line:PPERUpdateTheta}
	        Update $\kappa$ and then $\tilde p$ via \Eqref{eq:AdaClip update rule}; then, set $\langle p_\mathsf{min}, p_\mathsf{max} \rangle \gets \langle \rho_\mathsf{min}, \rho_\mathsf{max} \rangle \cdot \tilde p$\; \label{line:PPERDeltaAvg}
        }
	    \vspace{-1.5pt}        
        $\theta^- \gets \theta$ \textbf{if} $t \text{ mod } t_\mathsf{target} = 0$\; \label{line:PPERUpdateThetaMinus}
    }
    \vspace{-2.5pt}    
    \end{spacing}
\end{algorithm}

\vspace{-0.5em}

\subsection{PPER, A Stable Priority Updater}

\vspace{-0.5em}

Our proposal, PPER, combines the above three countermeasures, each of which increases sample diversity (i.e., making the distribution $\eDist (i)$ away from $\indicator(p_i = p_N) \slash N'$ but towards $1/N$ in a degree) and thus contributes to stability of the prioritized DQN. 
The pseudo-code of PPER is provided in Algorithm~\ref{alg:proposed algorithm}. As PPER inherits the features and properties of the three countermeasures, it is expected that PPER can best stabilize the DQN although not best performing, as shown in Figure~\ref{fig:ablation study:training curves}, due to the loss of prioritization at the cost of increasing sample diversity.

To quantify stability for each case and PPER, we evaluated the $\mathrm{normalized\_max\_forget}$ measure:
\begin{align}
    \mathrm{normalized\_max\_forget} := 
    \frac{\mathrm{score}(t_{T^*}) - \min_{t_{T^*} \leq t \leq T_\mathsf{max}} \mathrm{score}(t)}{\mathrm{score}(t_{T^*}) - \min_{1 \leq t \leq T_\mathsf{max}} \mathrm{score}(t)}
    \in [0, 1],
	\label{eq:normalized max forget}
\end{align}
where $\mathrm{score}(t)$ is the game score on a learning curve in Figure~\ref{fig:ablation study:training curves},  $t_T := \textstyle\Argmax_{1 \leq t \leq T} \mathrm{score}(t)$ denotes the first time at which the score becomes the maximum within the horizon $T \leq T_\mathsf{max}$, and $T^*$ is the horizon $T$ at which the $\mathrm{forget}(T) := \mathrm{score}(t_T) - \textstyle \min_{t_T \leq t \leq T_\mathsf{max}} \mathrm{score}(t)$ is maximized. The numerator of \Eqref{eq:normalized max forget} is equal to $\mathrm{forget}(T^*)$, the score the agent maximally forgets during the training; the denominator normalizes it by the gap between $\mathrm{score}(t_{T^*})$ and the minimum score. Note that $\mathrm{normalized\_max\_forget} = 1$ means that the agent at the end has forgot everything it had learned; it is zero when no forgetting observed on a training curve. For each case, the vertical dotted line in Figure~\ref{fig:ablation study:training curves} indicates the $\mathrm{forget}(T^*)$.

Figure~\ref{fig:ablation study:summary}(a) presents the normalized maximum forgets (\Eqref{eq:normalized max forget}) of PER, the three countermeasures (\tdinitnos, \tdclipnos, \tdprednos), their combinations (TDInitPred, TDInitClip, TDClipPred), and PPER, in percents. PER showed the severe forgetting in 4 out of 5 games (i.e., except Seaquest), whereas all the forgetting metrics of the proposed countermeasures, their combinations, and PPER, except \tdinit (Tennis; Zaxxon) and TDInitPred (Tennis), remain within around 25\% as shown in Figure~\ref{fig:ablation study:summary}(a). \emph{Overall, PPER is the minimum among them.} Except \tdinitnos, the average learning performance of all cases are better than PER as summarized in Figure~\ref{fig:ablation study:training curves}(b); \emph{PPER gives higher average scores than any other cases but TDClipPred.} We note that (i) the four (resp. two) best-performing ones in Figure~\ref{fig:ablation study:training curves}(b) all combine \tdclip (resp. \tdclip and \tdprednos); (ii)
combined with TDPred, the priorities at the end always show a variance-reduced Gaussian-like distribution (Appendix~\ref{appendix:ablation study}). More ablation study (e.g., TDClipPred/\tdinitnos) remains as a future work.

\begin{figure*}[t!]
\subfigure[Normalized Maximum Forgets]{\includegraphics[width=.5\columnwidth, height=.215\columnwidth]{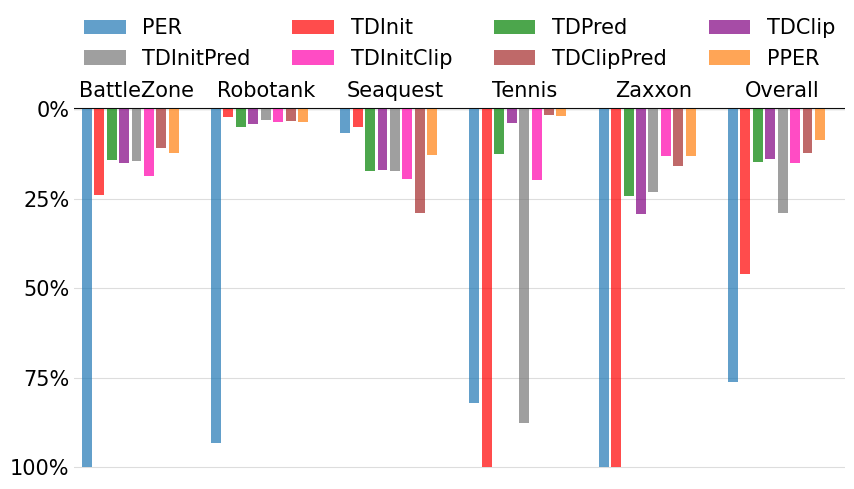}}
\subfigure[The Median of Mean Scores, PER-normalized]{\includegraphics[width=.5\columnwidth, height=.215\columnwidth]{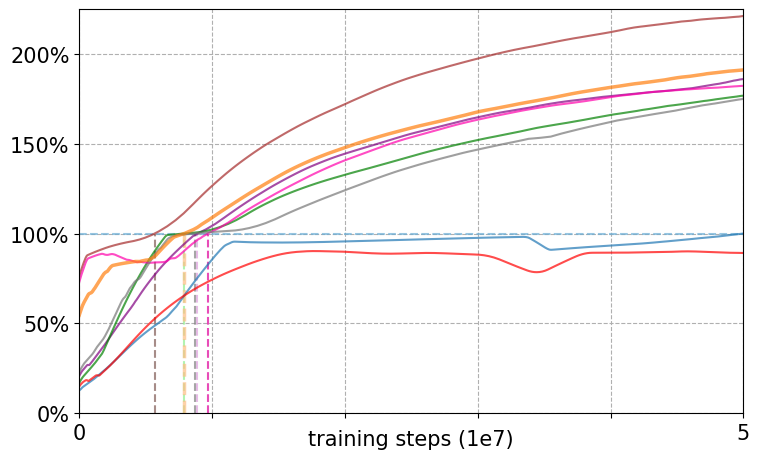}}
\vspace{-1.0em}
\caption{Summary of ablation study. (a) normalized maximum forgets for all cases; (b) the medium over 5-games of each average of the scores so far, normalized by PER.}
\vspace{-1.0em}
\label{fig:ablation study:summary}
\end{figure*}

\section{Experiments on Atari 2600 Environments}
\label{sec:Atari experiments}


Through Atari experiments, we evaluate stability and performance of PPER, relative to the original PER, the baseline. The result shows that PPER indeed improves the stability during training with the test score also higher than PER on the average. We designed the TDPred in a way that its convolutional layers, shared with the DQN, are followed by a fully-connected layer with $M$-neurons. Details on the experimental settings are provided in Appendix~\ref{appendix:experimental settings}. In the first place, we experimented the total 58 Atari games with $M = 512$ and the results are shown in Figures~\ref{fig:relative_scores},  \ref{fig:median_scores_main_experiment}(a), and \ref{fig:median_scores_main_experiment}(b).

\paragraph{Stability} We first see how much the stability of PPER is improved, by measuring 
\[
	\mathrm{relative\_stability\_score} := \mathrm{normalized\_max\_forget}_\mathsf{PER} - \mathrm{normalized\_max\_forget}_\mathsf{PPER}
\]
where $\mathrm{normalized\_max\_forget}_m$ is the \emph{normalized maximum forget}, given by \Eqref{eq:normalized max forget}, of the prioritization method $m$, with the $\mathrm{score}$'s coming from its training histories (see  Appendix~\ref{appendix:training curve}). If relative stability score is $1$, then it means that PER completely forgets whereas the PPER never, and vice versa if $\mathrm{relative\_stability\_score} = -1$. A positive relative stability score indicates that PPER less forgets than PER. As shown in Figure~\ref{fig:relative_scores}(a), PPER less forgets in 42 out of 58 games,  with $\mathrm{relative\_stability\_score} \approx 23.2\%$ on the average. In Appendix~\ref{appendix:variance reduction}, one can see that PPER indeed reduces the variance of the priorities, which can improve the sample diversity, hence stability.

\paragraph{Training Performance} 
Figure~\ref{fig:median_scores_main_experiment} shows (a) the median over 53 games of maximum PER-normalized scores achieved so far, and (b) the same as (a) but with the maximum replaced by average. Both represent the peak and the average performance over the training steps, respectively (see also \citealp[Figure~4]{PER2015}). It is clear that the average performance growth of PPER has been improved notably, as it reaches at the same level of PER's at 45\% of the total training steps approximately, but  the peak performance in Figure~\ref{fig:median_scores_main_experiment}(a) has not. From the training curves in Appendix~\ref{appendix:training curve}, we observed that \emph{this is mainly from the stabilized training of PPER}---PER can reach more or less the same maximum score as PPER (similar peak performance), but since PER forgets (sometimes severely), its average performance goes down.


\textbf{Evaluation via Testing} We evaluate the learned policies in \emph{relative test scores}, which measures the performance gap between PER and PPER with each game's difficulty taken into account: 
\[
	\mathrm{relative\_test\_score} := \frac{(\mathrm{score_\mathsf{PPER}} - \mathrm{score_\mathsf{PER}})}{\max \{ \mathrm{score_\mathsf{Human}}, \mathrm{score_\mathsf{PER}}\} - \mathrm{score_\mathsf{Random}}}.
\]
Here $\mathrm{score_\mathsf{PPER}}$ and $\mathrm{score_\mathsf{PER}}$ are corresponding test scores, $\mathrm{score_\mathsf{Human}}$ and $\mathrm{score_\mathsf{Random}}$ are the human- and random-play scores reported by \citet{DuelingDQN2015}. For $\mathrm{score_\mathsf{PPER}}$ and $\mathrm{score_\mathsf{PER}}$, we take the average score over 200 repetitions of the \emph{no-ops start} test with the best-performing model \citep{Atari2015, DDQN2016}. From Figure~\ref{fig:relative_scores}(b), we observed that policies are improved in PPER on 34 of 53 games. For those games, we found that its training curves are mainly either gradually increasing without, or less suffering from forgettings. Also, unlike PER, priority distributions of PPER are Gaussian-like, where priority outliers and explosion are removed. As we claimed above, such a difference is the key to balance priority and diversity in experience sampling, which eventually contributes to stabilized, efficient training (See Appendices~\ref{appendix:ablation study} for Gaussian-like priority distributions of PPER and \ref{appendix:raw test scores} for full test scores)


\textbf{Stability and Performance When $M$ is Smaller} 
In this second experiment, we trained the DQN with PPER over 15 Atari games 
under the same experimental settings as above, except that the number $M$ of the neurons in \tdprednos's fully-connected layer is reduced to 256, 128, 64 and 32.  Figures~\ref{fig:median_scores_main_experiment}(c) and (d) 
summarize training of those variations in median scores, which can answer ``No'' to: ``If the \tdpred becomes simpler, then will it terribly affects the performance?'' Indeed, Figure~\ref{fig:median_scores_main_experiment}(d) shows that for all such variations, the training efficiency has been improved on the average, even when $M = 32$, the minimum. Moreover, except for a few case(s) with $M = 32$,  the priority distributions are still Gaussian-like with no explosion or no significant outliers, 
meaning that PPER with much smaller TDPred can also enhance stability of training. The learning curves in Appendix~\ref{appendix:subsection:PPN with smaller layer} clearly show PPER in this case has higher scores and forgets less than PER.

However, the priority distributions of those variations are not identical, which potentially changes the level of diversity PPER pursues. In DemonAttack, for example, $M = 32$ performed best among all the variations probably because its weakest diversity pursuing behavior helps it focus priority the most. Consequently, $M$ is a hyper-parameter of PPER, which trades off prioritization, sample diversity, and complexity. Regarding complexity, the training cost of \tdpred reported in Table~\ref{tab:computational_cost} (Appendix~\ref{appendix:subsection:PPN with smaller layer}) is about 10.9\% of DQN's for $M = 512$ and 6.6\% for $M = 128$, which are significantly less than DQN's, thanks to our design choice. See Appendix~\ref{appendix:subsection:PPN with smaller layer} for details of the experiment.


\vspace{1em}

\begin{figure*} [h!]
\subfigure[Relative stability scores (58 games)]{\includegraphics[width=.5\columnwidth]{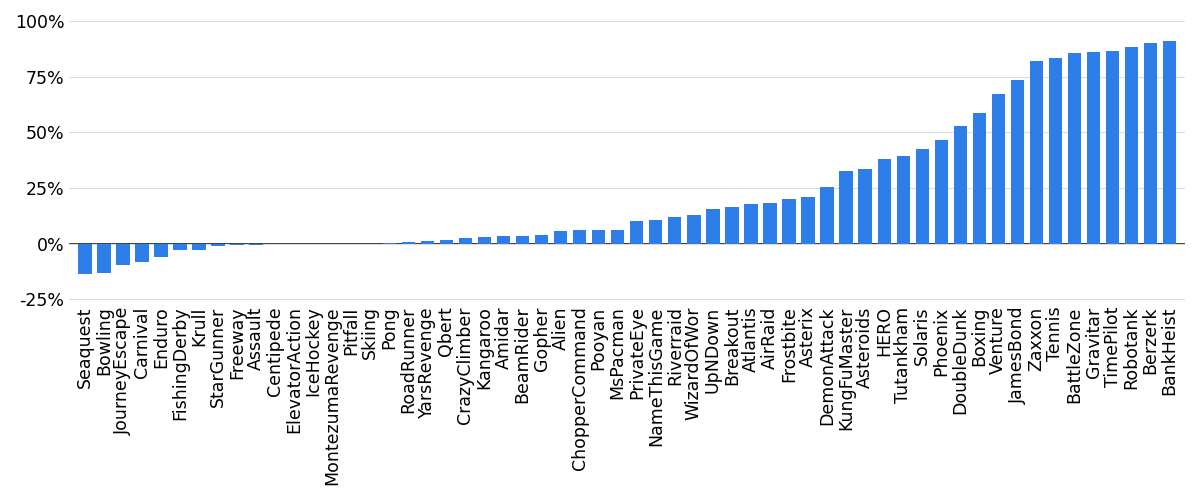}}
\subfigure[Relative test scores (53 games)]{\includegraphics[width=.5\columnwidth]{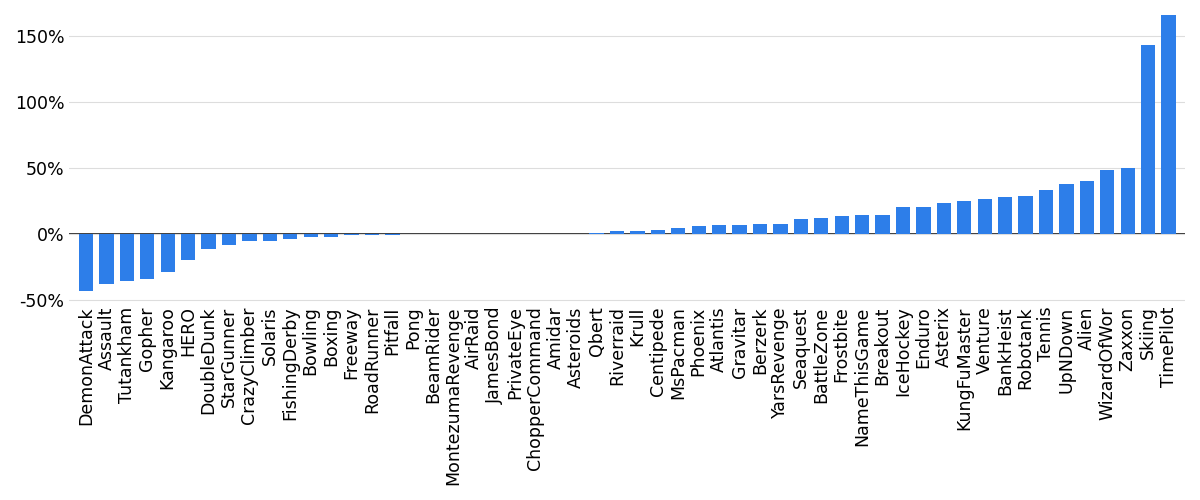}}
\vspace{-1em}
\caption{Stability and test scores of PPER relative to PER in percents, trained up to 50M frames.}	
\label{fig:relative_scores}
\end{figure*}

\vspace{1em}

\begin{figure*}[h!]
\subfigure[Median Max Score]{\includegraphics[width=.245\columnwidth, height=.18\columnwidth]{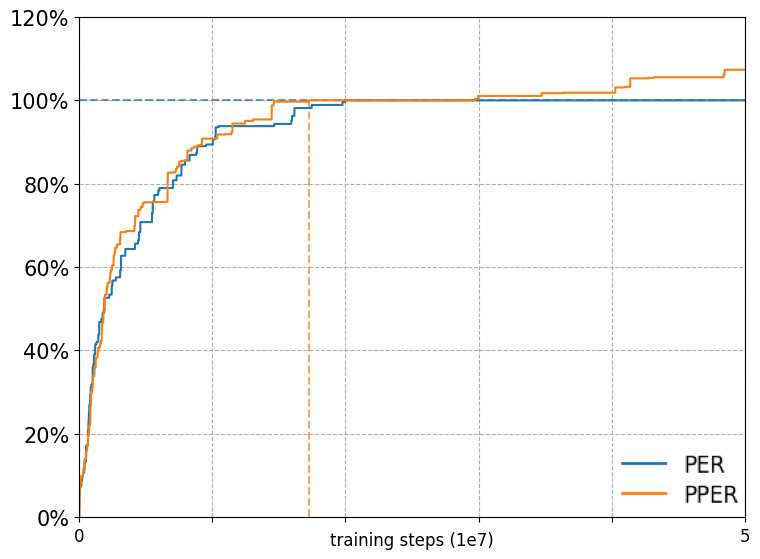}}
\subfigure[Median Avg Score]{\includegraphics[width=.245\columnwidth, height=.18\columnwidth]{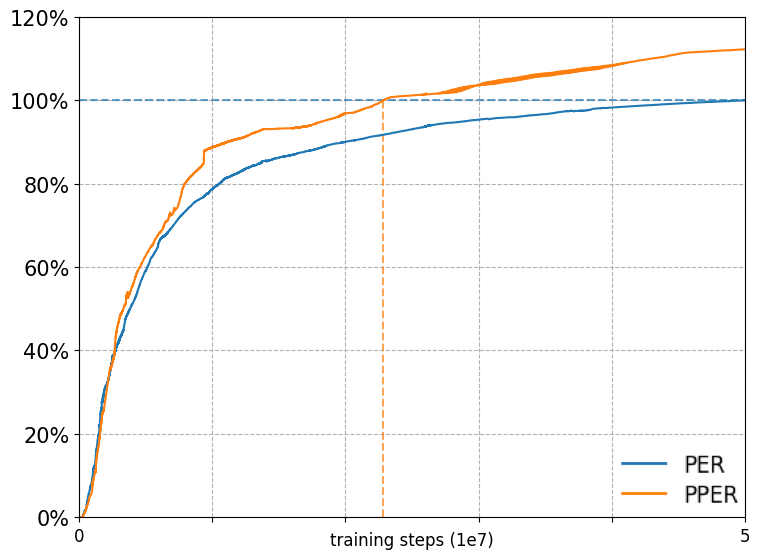}}
\subfigure[Median Max Score]{\includegraphics[width=.245\columnwidth,
height=.18\columnwidth]{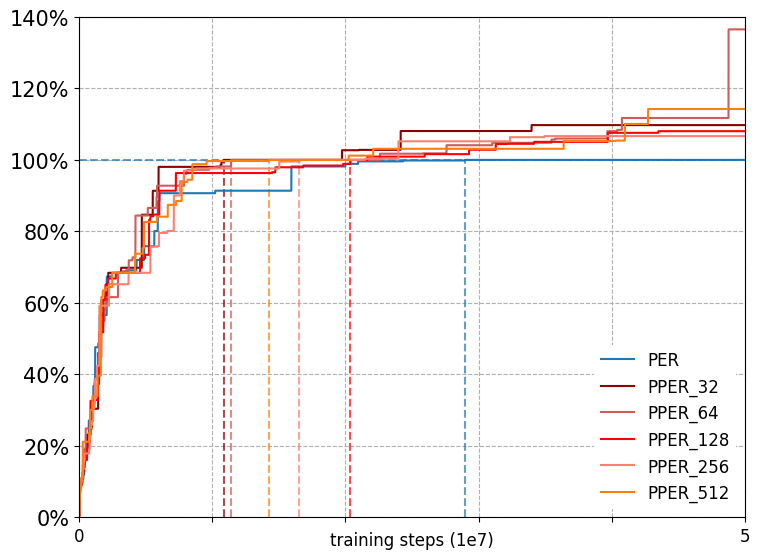}}
\subfigure[Median Avg Score]{\includegraphics[width=.245\columnwidth, height=.18\columnwidth]{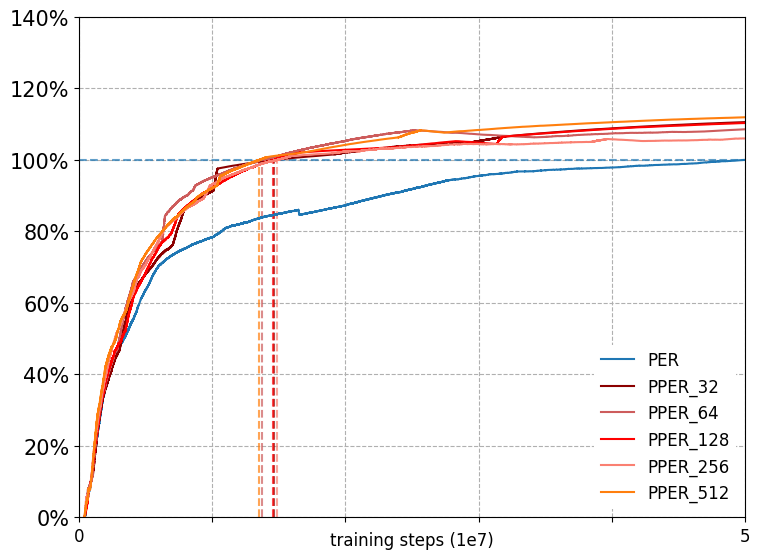}}
\vspace{-1em}
\caption{Summary of training over 53 games (a, b) / over 15 games with smaller $M$'s (c, d).}
\label{fig:median_scores_main_experiment}
\end{figure*}

\section{Conclusions}

This paper proposed predictive PER (PPER) which combines three countermeasures (\tdinitnos, \tdclipnos, and \tdprednos) designed to effectively remove the outliers, spikes, and explosions of priorities for balancing diversity of samples over prioritization. Ablation and full experimental studies with Atari games justified our claim: a certain level of diversity should be kept for maintaining stability. The experimental results showed PPER improves both performance and notably, stability. The authors believe that the proposed methods can be served as ingredients for building stable DRL systems. 

\subsubsection*{Acknowledgments}
The authors gratefully acknowledge the support of the Japanese Science and Technology agency (JST) ERATO project JPMJER1603: HASUO Metamathematics for Systems Design.

\bibliography{references}
\bibliographystyle{iclr2021_conference}

\newpage
\appendix

\renewcommand\thelemma{\thesection.\arabic{lemma}}
\renewcommand\thetheorem{\thesection.\arabic{theorem}}
\renewcommand\theequation{\thesection.\arabic{equation}}
\renewcommand\thetable{\arabic{table}}
\setcounter{equation}{0}

\section{Related Works.}
\label{appendix:related works}

\paragraph*{ER and Its Extensions} ER has been employed in deep Q-learning \citep{Atari2013,Atari2015}, its extensions \citep{DuelingDQN2015,DDQN2016}, and the other DRL methods \citep{DDPG2015,ACER2016}, as an essential ingredient for providing un-correlated experiences and improving stability and sample efficiency. Several state-of-the-art extensions exist such as hindsight ER \citep{HER2017}, curiosity-driven ER \citep{Curiosity-drivenER2019}, ER optimization \citep{ERO2019}, distributed ER \citep{DER2015}, and our focus, PER \citep{PER2015}. Other researchers have also investigated the effects of ER, and its components, on (D)RL \citep{EPStudy2015,ExperienceSelection2018,Harm2015,Zhang2017,Pan2018}.
 
\paragraph*{PER and Underlying Ideas} Many researchers have utilized PER in DRL methods --- DQN and its extensions \citep{PER2015,DuelingDQN2015,Rainbow2018,DQNwithDemonstration2018} and others \citep{ACER2016,DDPG-with-PER2017,DDPG-with-Demonstration2017} --- for more efficient experience sampling in replay. The underlying idea of PER is known as prioritized sweeping, which prioritizes the samples in proportion to their absolute TD errors. The idea was originally applied to reinforcement planning, e.g., Dyna-Q \citep{Sutton2018}, and value iteration \citep{PS1993,GPS1998}. 

\paragraph*{Extensions of PER} PER has been extended by storing and sampling the experiences from either the demonstration \citep{DDPG-with-Demonstration2017,DQNwithDemonstration2018}, or the distributed actors \citep{DPER2018}. \citet{SPER2019} proposed the Prioritized Sequence ER (PSER), which also extends the idea of PER. All of those extensions were provided with their promising experimental performance on Atari and/or continuous control tasks. PSER additionally updates all the priorities along the backward-in-time trajectory of each transition currently sampled. When updating the priorities, PSER tries to prevent the priorities from going suddenly small by slowly decaying them. In PPER, the same purpose is served by \tdpred and lower-clipping in \tdclip (Sections~\ref{subsection:statistical priority clipping} and \ref{subsection:PPN}).

\section{Details on TDClip}
\label{appendix:details on SPC}

This appendix describes the details of TDClip: (i) the update rule (\Eqref{eq:AdaClip update rule}) of the estimate~$\tilde p$ of $\sum_{i=1}^N |\delta(e_i)| / N$, and (ii) the hyper-param setting $\langle \rho_\mathsf{min}, \rho_\mathsf{max} \rangle = \langle 0.12, 3.7 \rangle$. 

\subsection{Estimating the Average Absolute TD-error in $\mathscr{D}$}
\label{appendix:est of average absolute TD-error}

The update rule~(\Eqref{eq:AdaClip update rule}) of the estimate $\tilde p$ can be rewritten as
	\begin{align}
			\kappa_{n+1} = \lambda \cdot \kappa_n + 1
	\quad\text{and then}\quad
			{\tilde p}_{n+1} = {\tilde p}_n + \big ({\hat \mu}_{n+1} - {\tilde p}_n \big ) / 	\kappa_{n+1},
		\label{eq:appendix:AdaClip update rule}
	\end{align}
	where $\kappa_n$ and ${\tilde p}_n$ denote $\kappa$ and $\tilde p$ at the $n$th batch update, respectively; $\lambda \in (0, 1)$ is the forgetting factor; 
	$
		\textstyle {\hat \mu}_n := \Delta_a
	$ 
	is the estimate of $\mu_n := \sum_{i = 1}^{N} |\delta(e_i)| / N$, the true average absolute TD  error over the replay memory $\mathscr{D}$, at the $n$th batch update. Note that $\Delta_a = \sum_{k=1}^{K} w(i_k) \cdot |\delta(e_{i_k})| / K$, where $w(i) := (N \cdot \eDist(i))^{-1}$, and that by the initializations of $\kappa$ and  ${\tilde p}$, we have $\kappa_0 = {\tilde p}_0 = 0$.

	Note that the TD-error $\delta$, the IS weight $w$, and the replay memory $\mathscr{D}$ are typically varying over the batch update step $n$ since at every batch step, updated are the weights $\theta$ and $\theta^-$ of the DQNs and the experiences and priorities in $\mathscr{D}$. For notational simplicity, however, we made their dependency on the batch step $n$ implicit.
\begin{lemma}
	\label{lemma:1}
	For each $n \in \mathbb{N}$, ${\hat \mu}_n$ is an unbiased estimator of $\mu_n$. 
\end{lemma}
\begin{proof}
	Fix $n \in \mathbb{N}$ and note that by the IS weight $w(i) := (N \cdot \eDist(i))^{-1}$,
	\[
		\mathbb{E}_{i \sim \eDist(\cdot)} \Big [ w(i) \cdot |\delta(e_i)| \Big ] = \sum_{i = 1}^N \eDist(i) \cdot w(i) \cdot |\delta(e_i) | = \frac{1}{N} \sum_{i = 1}^N |\delta(e_i) | = \mu_n.
	\]
	Since each index $i_k$ in the batch is sampled from $\eDist(\cdot)$, we have
	\begin{align*}
		\mathbb{E} \big [{\hat \mu}_n \big ] 
		= \sum_{k=1}^K \mathbb{E}_{i_k \sim \eDist(\cdot)} \Big [ w(i_k) \! \cdot |\delta(e_{i_k})| \Big ] / K
		= \mu_n \cdot \sum_{k=1}^K 1/K = \mu_n,
	\end{align*}
	which completes the proof.
\end{proof}

\begin{lemma}
	\label{lemma:2}
	For each $n \in \mathbb{N}$, ${\tilde p}_n = \dfrac{{\hat \mu}_n + \lambda \cdot {\hat \mu}_{n-1}  + \cdots + \lambda^{n - 1} \cdot {\hat \mu}_1}{1 + \lambda + \cdots + \lambda^{n - 1}} = \dfrac{ \sum_{m=0}^{n-1} \lambda^{m} \cdot {\hat \mu}_{n - m}}{\sum_{m=0}^{n-1} \lambda^m}$.
\end{lemma}
\begin{proof}
First, note that $\kappa_n = \sum_{m=0}^{n-1} \lambda^m$ since $\kappa_1 = 1$ and 
\[
	\kappa_n = \lambda \kappa_{n-1} + 1 = \lambda^2 \kappa_{n-2} + (\lambda + 1) = \cdots = \lambda^{n-1} \kappa_1 + (\lambda^{n-2} +  \lambda^{n-3} + \cdots + \lambda + 1).
\]
Hence, extending ${\tilde p}_n$ in a similar manner, we obtain:
\begin{align*}
	\kappa_n \cdot {\tilde p}_{n} 
	&= \kappa_n \cdot {\tilde p}_{n-1} + ( {\hat \mu}_{n} - {\tilde p}_{n-1} )
	\\
	&= {\hat \mu}_n + ( \kappa_n - 1 ) \cdot {\tilde p}_{n-1}
	\\
	&= {\hat \mu}_n + \lambda \cdot \kappa_{n-1} \cdot {\tilde p}_{n-1}
	\\
	&= {\hat \mu}_n + \lambda \cdot {\hat \mu}_{n-1} + \lambda^2 \cdot \kappa_{n-2} \cdot {\tilde p}_{n-2}
	\\
	&= {\hat \mu}_n + \lambda \cdot {\hat \mu}_{n-1} + \lambda^2 \cdot {\hat \mu}_{n-2} + \lambda^3 \cdot \kappa_{n-3} \cdot {\tilde p}_{n-3}
	\\
	&\;\,\vdots
	\\
	&= {\hat \mu}_n + \lambda \cdot {\hat \mu}_{n-1} + \lambda^2 \cdot {\hat \mu}_{n-2} + \lambda^3 \cdot {\hat \mu}_{n-3} + \cdots +\lambda^{n-1} \cdot \kappa_1 \cdot {\tilde p}_{1}.
\end{align*}	
This completes the proof since $\kappa_n = \sum_{m=0}^{n-1} \lambda^m$, $\kappa_1 = 1$, and ${\hat \mu}_1 = {\tilde p}_1$.
\end{proof}

\begin{theorem}
	\label{thm:statistical priority clipping}
	For each $n \in \mathbb{N}$, 	$\mathbb{E}[ \tilde p_n ] = \dfrac{ \sum_{m=0}^{n-1} \lambda^{m} \cdot {\mu}_{n - m}}{\sum_{m=0}^{n-1} \lambda^m}$.
\end{theorem}
\begin{proof}
		By Lemmas~\ref{lemma:1} and \ref{lemma:2}, it is obvious that for each $n \in \mathbb{N}$,
		\[
			\mathbb{E}[{\tilde p}_n] = \dfrac{ \sum_{m=0}^{n-1} \lambda^{m} \cdot \mathbb{E} [{\hat \mu}_{n - m}]}{\sum_{m=0}^{n-1} \lambda^m} = \dfrac{ \sum_{m=0}^{n-1} \lambda^{m} \cdot \mu_{n - m}}{\sum_{m=0}^{n-1} \lambda^m}
		\]
		which completes the proof.
\end{proof}

By Theorem~\ref{thm:statistical priority clipping}, if the priority mean is stationary, i.e., if 
$
    \mu := \mu_1 = \mu_2 = \cdots = \mu_n,
$
then ${\tilde p}_n$ serves as an unbiased estimator of $\mu$. In general, by Theorem~\ref{thm:statistical priority clipping}, $\mathbb{E}[{\tilde p}_n]$ is the mean of the mixture distribution $f$ of absolute TD errors $|\delta|$'s over replay memories~$\mathscr{D}_k$'s, at batch steps $k = 1, 2, \cdots, n$:
\[
	f(x) := \dfrac{ \sum_{m=0}^{n-1} \lambda^{m} \cdot f_{n - m}(x)}{\sum_{m=0}^{n-1} \lambda^m}
\]
where $f_k(x) := \sum_{e \in \mathscr{D}_k} \indicator(x = |\delta(e)|) / N_k$ is the distribution of $|\delta|$ over $\mathscr{D}_k$ and $N_k := |\mathscr{D}_k|$ is the length of the replay memory $\mathscr{D}_k$, both at batch step $k$.

Note that each ${\hat \mu}_n$ is obtained with ordinary importance sampling (OIS) technique \citep{IS2014}, which can be replaced by weighted importance sampling (WIS). However, our simple tests with the estimator (\Eqref{eq:appendix:AdaClip update rule}) on stationary distributions suggested using OIS rather than WIS since the former performed better than the latter in those tests, without increasing the variance. Of course, it is open to investigate more on better design choices of the estimate ${\hat \mu}_n$ for TDClip.


\subsection{Selection of the Hyper-parameters}

At each batch step $n$, the hyper-parameters $\rho_\mathsf{min}$ and $\rho_\mathsf{max}$ are used to determine the lower- and upper-clipping thresholds from the estimated absolute TD-error average ${\tilde p}$ over $\mathscr{D}$ as $p_\mathsf{min} = \rho_\mathsf{min} \cdot {\tilde p}$ and $p_\mathsf{max} = \rho_\mathsf{max} \cdot {\tilde p}$, respectively. To select those hyper-parameters, assume simply that the distribution of the TD error $\delta$ over $\mathscr{D}$ is zero-mean Gaussian with the standard deviation $\sigma$. Then, the absolute TD error $|\delta|$ has a \emph{folded Gaussian distribution} whose mean is given by $\mathbb{E}[|\delta|] = \sigma \sqrt{2 / \pi}$. 

Considering $\xi \sigma$-point ($\xi > 0$) of the original Gaussian distribution $\mathscr{N}(0, \sigma)$, we determine the multiplier $\rho > 0$ satisfying 
	\begin{align*}
		\xi \cdot \sigma = \rho \cdot \mathbb{E}[ | \delta | ] \; 
		(= \rho \cdot \sigma \sqrt{2/\pi})		
	\end{align*}
under the aforementioned Gaussian assumption: $\delta \sim \mathscr{N}(0, \sigma)$. Rearranging the equation, we finally obtain the equation $\rho = \xi \sqrt{\pi / 2}$ for the hyper-parameters $\rho_\mathsf{min}$ and $\rho_\mathsf{max}$. For the upper-clipping, any TD error $\delta$ above the $3\sigma$-point (i.e., $\xi = 3$) in its magnitude is considered as an outlier and thus clipped to the $3\sigma$-point. So, we decide $\rho_\mathsf{max} = 3.7 \approx 3 \sqrt{\pi/2}$. For the lower-clipping, we decided to clip any absolute TD-errors less than the $0.1\sigma$-point (i.e., $\xi = 0.1$), so $\rho_\mathsf{min} = 0.12 \approx 0.1 \sqrt{\pi/2}$.

One technique to decide $p_\mathsf{min}$ and $p_\mathsf{max}$ would be to estimate the variance of $|\delta|$ and actively use it. However, in such a variant, the variance estimate at some point can possibly (i) be very small, making PPER too uniform, or (ii) even become negative. Moreover, variance of the TDPred output $\smash{\hat \delta}$ is much reduced from the TD error $\delta$ 
(see Section~\ref{subsection:PPN} and Appendix~\ref{appendix:variance reduction}), which increases the possibility of having the two cases (i) and (ii)  above. For such reasons, we observed no performance improvement with such a variant based on variance-mean estimation.



\section{Atari 2600 Experimental Settings}
\label{appendix:experimental settings}

In Section \ref{sec:Atari experiments}, we experimentally compared the performance of PER and PPER applied to the double DQN agents \citep{DDQN2016} with dueling network architecture \citep{DuelingDQN2015}. Here, We first introduce experimental settings that are common for both PER and PPER.

\paragraph{Input Preprocessing}
Each input to the DQN is a sequence of preprocessed observation frames with a size of $84 \times 84 \times 4$. The preprocessing has done in exactly the same way to the nature DQN: gray-scaling, max operation with the previous observations, and resizing \citep{Atari2015}. 

\paragraph{Hidden Layers}
It is as same as the nature DQN \citep{Atari2015}. Specifically,  our DQN contains an overall $4$ hidden layers. Convolutional are the first $3$ hidden layers that have $32$ of $8 \times 8$ filters with stride $4$, $64$ of $4 \times 4$ filters with stride $2$, and $64$ of $3 \times 3$ filters with stride $1$, respectively. The last hidden layer is ReLU and fully-connected with $512$ neurons, whose input is the flattened output of the last convolutional layer. 

\paragraph{Output Layer: Dueling Network}
Instead of using a single output layer, we adopted the dueling network \citep{DuelingDQN2015}, which estimates the state and advantage values in a split manner and combines them to output and better approximate the Q-values.

\paragraph{Proportional-based PER}
We adopted exactly the same exponents $\alpha$ and $\beta$ to proportional-based PER \citep{PER2015}. We fixed $\alpha = 0.6$ for prioritization (\Eqref{eq:probPrioritySampling}) but linearly scheduled $\beta$, starting from $0.4$ and reaching $1.0$ at the end of the training, for IS weights compensation. Note that the learning rate $\eta_\mathsf{q}$ of the DQN should be regulated by $1/4$, since the typical gradient magnitudes are much larger than the DQN without PER \citep{PER2015}.

\paragraph{Training the DQN Agent}
We trained the dueling double DQN agents up to the $50$ million frames, with frame skipping of $4$, where each episode starts with the no-ops start, up to 30 steps. For the exploration, we used $\epsilon$-greedy policy, where $\epsilon$ is linearly scheduled from $1.0$ to $0.1$ up to the first $1$ million frames, and then fixed to $0.1$ thereafter. Each agent starts learning once its replay memory is filled with $50,000$ experiences. The capacity $N$ of the replay memory $\mathscr{D}$ is $10^6$. For every $t_\mathsf{replay} = 4$ steps, the agent first samples minibatch of $K = 32$ experiences from the replay memory and then updates its parameters. The total number of training steps is $T_\mathsf{max} = 50$ millions.

The networks are trained with the Adam optimizer and the gradient clipping of $\pm 10$; the Huber loss is minimized for stability reason. The discount factor $\gamma = 0.99$. Note that we set the learning rate of the DQN's optimizer as $\eta_\mathsf{q} = 0.00025 / 4$. We empirically observed that our simulation settings---the Huber loss, the Adam optimizer with the gradient clipping, and regulated $\eta_\mathsf{q}$---imply rather stable hence better training results in PER, for some games.

Finally, the target network is updated for every 10,000 steps by copying the parameters of the DQN. Also note that we clipped the rewards to fall into the range of $[-1, 1]$ to share all the same hyperparameters over diverse Atari games with different score scales, as done by \citet{Atari2015}.

\paragraph{Priority} To ensure that every experience in the replay memory has some positive probability being sampled, we used the priority ``$p + \varepsilon$'' rather than $p$ itself, with $\varepsilon = 10^{-6}$, and  $p_t + \varepsilon$ rather than $p_t$ whenever \tdinit is combined.

\subsection{Additional Experimental Settings for PPER}

\paragraph{\tdpred Architecture}
The same convolutional layers of the DQN are used in PPER, two in parallel (one for the DQN and the other for the \tdprednos), for extracting the feature vectors from two consecutive observation frames $s$ and $s'$. Then two flattened feature vectors for both $s$ and $s'$, the reward $r$ and the action $a$ are passed to one ReLU, fully-connected layer of the \tdpred followed by a $1$-linear fully-connected layer resulting in one scalar output ${\hat \delta}(s,a,r,s')$ of the \tdprednos. 

\paragraph{Training the \tdpred}
Since the \tdpred reuses the convolutional output of the DQN, the update of the \tdpred simply means the update of its fully-connected part only, using the sampled mini-batch. We synchronized this update with that of the DQN: whenever the DQN computes the lose for a sampled mini-batch, the \tdprednos's loss is also computed as the Huber loss, between corresponding TD-error $\delta$ and the \tdpred output $\hat \delta$. Here, the optimization of the \tdpred is performed by an Adam optimizer with a different learning rate $\eta_\mathsf{p}$ than $\eta_\mathsf{q}$, and without the gradient clipping. We empirically chose $\eta_\mathsf{p} = 1.5~\eta_\mathsf{q}$.

\paragraph{Statistical Priority Clipping}
The derivation of hyper-parameters $\rho_\mathsf{min}$ and $\rho_\mathsf{max}$ for statistical priority clipping is given in Appendix~\ref{appendix:details on SPC}. We selected $\langle \rho_\mathsf{min}, \rho_\mathsf{max} \rangle = \langle 0.12, 3.7\rangle$ so that priorities are clipped to the range $\approx [0.1 \sigma_n, 3 \sigma_n]$. The forgetting factor $\lambda = 0.9985$ is used for all experiments.

\subsection{Designing the \tdpred}
\label{appendix:subsection:PPN design}
The \tdpred takes an experience tuple $e =(s, a, r, s')$ as its input and outputs a real value ${\hat \delta}(e)$ that approximates the TD-error $\delta(e)$. Here the experiences are also shared by both \tdpred and DQN. The architecture introduced above was adopted in our Atari experiments since it was best-performing over all our trials, with reasonable computational costs. Two main features of our \tdpred design are as follow.

\paragraph{Convolutional Layers Required}
There are many possible ways to design the \tdpred. One example is to compose the \tdpred with fully-connected layers only. However, we empirically concluded that such an approach might be too naive; the \tdpred without any convolutional layers seemed to be not able to learn complicated behaviors of ${\delta}(e)$. We hypothesize that the \tdpred must contain convolutional layers, as much as the DQN has, to close the gap between the \tdpred and the DQN.

\paragraph{Reusing Convolutional Layers}
Reusing convolutional layers of the DQN is also the key, not just because it reduces the computational cost of the updates, but it establishes a common ground of the DQN and the \tdprednos. Otherwise, the \tdpred will optimize its own convolutional filters independently, which potentially results over-approximation of itself. We observed such over-approximation when we do not reuse convolutional layers in the same \tdpred architecture previously described.

\clearpage

\section{More on Ablation Study}
\label{appendix:ablation study}

Here, we show the details on the ablation study, i.e., the study on the individual three countermeasures (\tdinitnos, \tdclipnos, \tdprednos) and their combinations (TDInitPred, TDInitClip, TDClipPred). The following is the learning curves for the latter (for the former, see Figure~\ref{fig:ablation study:training curves}).

\vspace{2em}

\begin{figure*}[h!]
\includegraphics[width=\columnwidth]{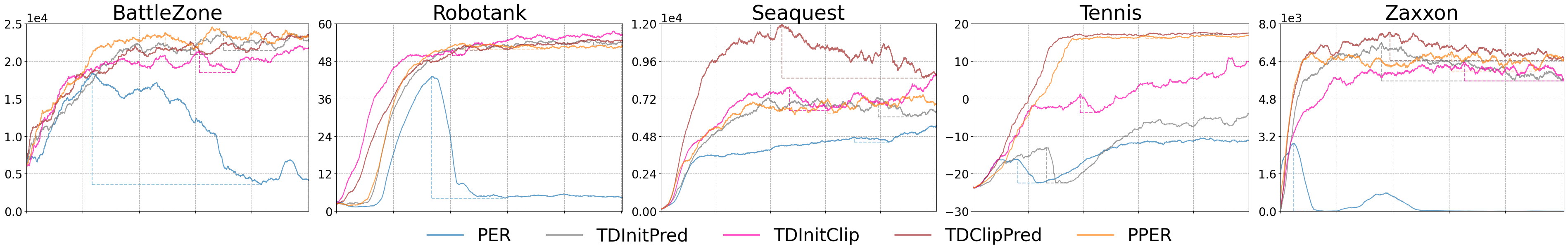}
\caption{Traing curves of PER, PPER, and the three combinations (TDInitPred, TDInitClip, and TDClipPred). Except TDInitPred in Tennis, all the proposed methods overcome the severe forgetting that PER has experienced in the games.}
\label{fig:appendix:ablation:combinations}
\vspace{-1em}
\end{figure*}

\vspace{1.5em}

The following figures show the statistically stable behavior of \tdprednos. See also Figure~\ref{fig:ablation:various figs}(d) and related discussions in Section~\ref{section:PPER}.

\vspace{1.5em}

\begin{figure*}[h!]
\subfigure[Magnitude of average]{\includegraphics[width=\columnwidth]{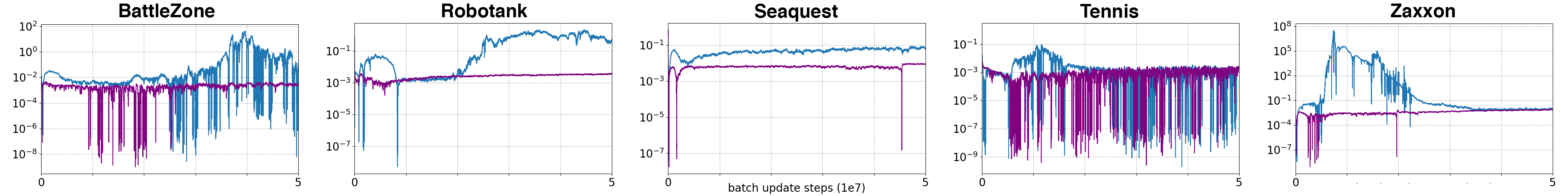}}
\subfigure[Standard Deviation]{\includegraphics[width=\columnwidth]{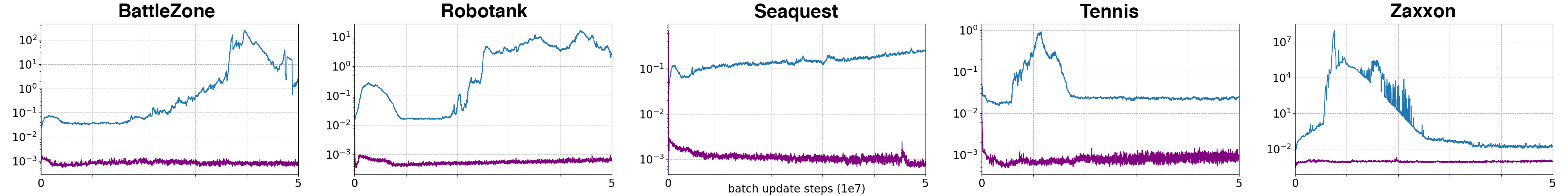}}
\caption{Statistics of the batch mean of $\delta$ (PER, {\color{blue}blue}) and $\hat \delta$ (\tdpred, {\color{violet}violet}), up to $T_\mathsf{max} = 50$M training steps. In the figures, (a) magnitude of average and (b) standard deviation for the \tdpred are much stabilized, with (b) much lower variance than PER, showing the stable behavior that can prevents priority explosion.}
\label{fig:appendix:batch tde statistics}
\end{figure*}

\vspace{1.5em}

\begin{figure*}[b!]
\subfigure[Priority distributions of TDPred]{\includegraphics[width=\columnwidth]{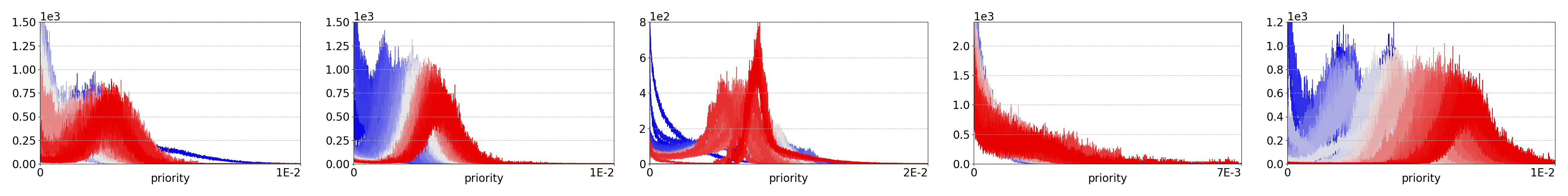}}
\subfigure[Priority distributions of PPER]{\includegraphics[width=\columnwidth]{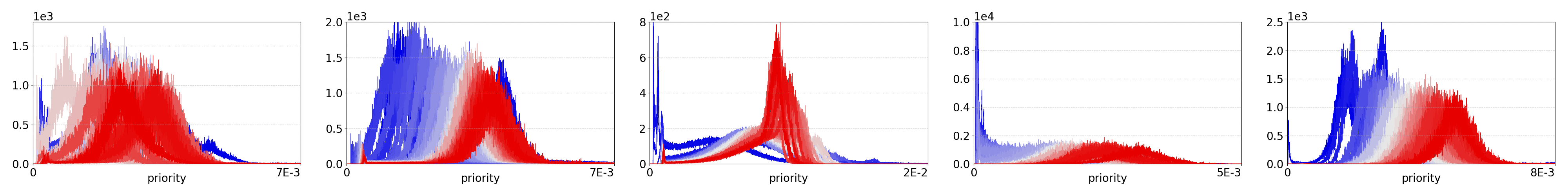}}
\caption{Priority distribution changes over the training for 5 Atari games (from left: BattleZone, Robotank, Seaquest, Tennis and Zaxxon). The colors of distributions change from {\color{blue}blue} to {\color{red}red} as the training steps increase up to $T_\mathsf{max} = 50$M.}
\label{fig:appendix:priority distribution ablation}
\end{figure*}

\clearpage

\section{Training Curves}
\label{appendix:training curve}

\begin{figure*}[h!]
\includegraphics[width=\columnwidth, height=20cm]{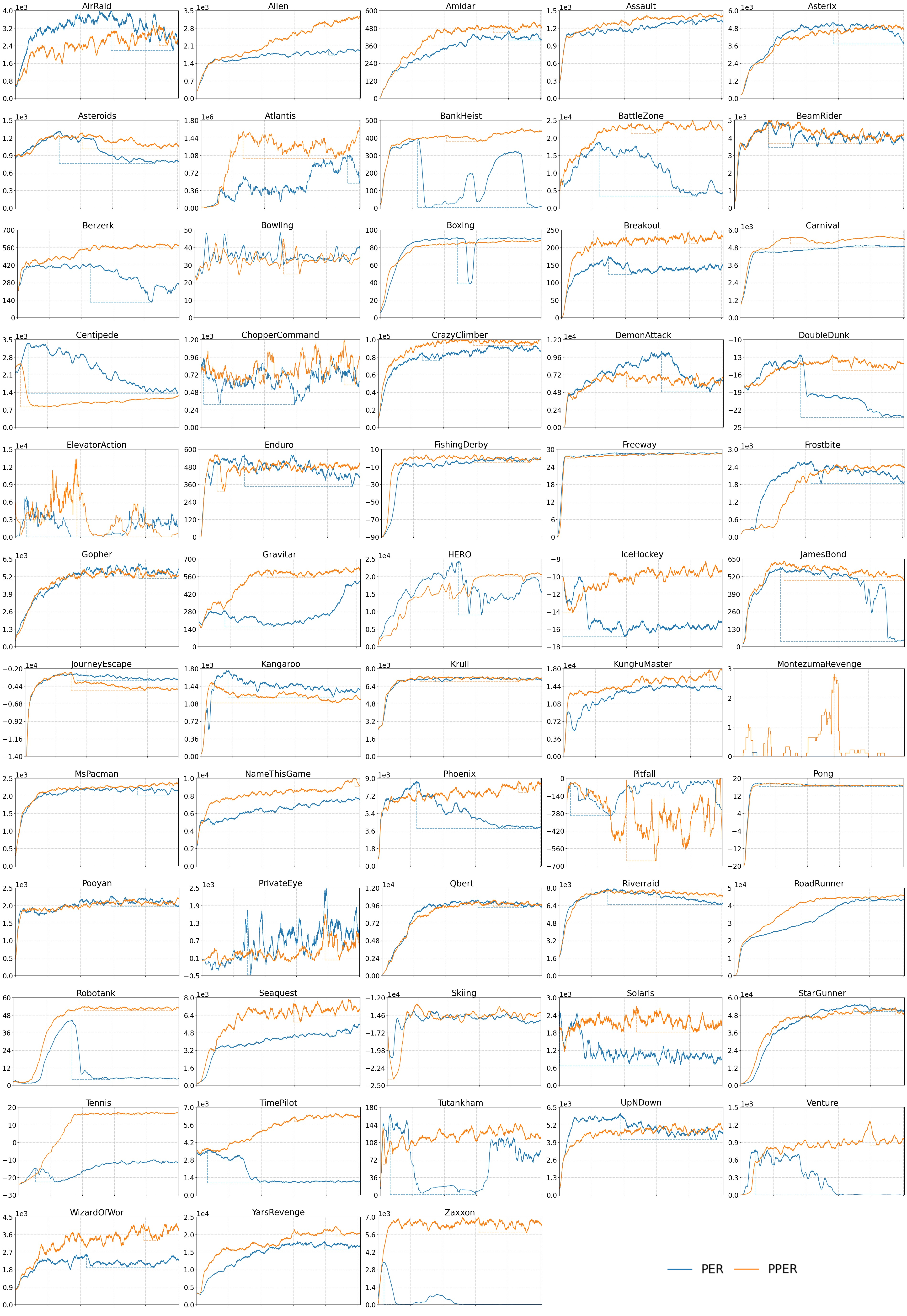}
\caption{Training curves up to 50 million frames, in raw scores, for 58 games of Atari 2600. In most of the cases, PPER either reaches outperforming performance of PER in the early stage and then tends to be stabilized, or keeps/improves the score steadily over the entire training, with no or less forgetting --- showing the stable learning performance of PPER. The vertical dotted line of each learning curve represents $\mathrm{forget}(T^*)$, \emph{the maximum score forget}.}	

\label{fig:full training curves}
\vspace{-1em}
\end{figure*}

\clearpage

\section{Raw Test Scores of 58 Atari 2600 Games}
\label{appendix:raw test scores}
\vspace{-0.5em}
\begin{table} [h!]
	\footnotesize
	\label{tab:scores}
	\centering
	\begin{tabular}{l||r|r|r|r||r}
	    \toprule 
	    \\[-1.3em]    
		 \multirow{2}{*}{Games} 
		 & 
		 \multicolumn{4}{c||}{Scores} & 
		 \multirow{2}{*}{\makecell[c]{Relative \\ Scores}	}
		 \cr
		 \cline{2-5}
		 & $\;\;$Random$\;\;$ & $\;\;\;$Human$\;\;\;$ & $\;$PER (50M)$\;$ & PPER (50M) &  \\
	    \hdrule
	AirRaid        			& -			& -    		& 5817.9                     & \textbf{6281.8}            &      -      \\
	Alien            		& 227.8    	& 7127.7  	& 2220.5                     & \textbf{5009.5}            & 40.42  		\\
	Amidar           		& 5.8      	& 1719.5  	& 573.1                      & \textbf{574.1}             & 0.05  		\\
	Assault          		& 222.4    	& 742.0   	& \textbf{3656.0}            & 2371.8                     & -37.40 		\\
	Asterix          		& 210.0   	& 8503.3  	& 5904.0                     & \textbf{7835.5}            & 23.28  		\\
	Asteroids        		& 719.1    	& 47388.7 	& 1232.8                     & \textbf{1261.5}            & 0.06  		\\
	Atlantis         		& 12850    	& 29028.1 	& 657339.0                   & \textbf{699726.0}          & 6.57  		\\
	BankHeist        	    & 14.2     	& 753.1   	& 712.3                      & \textbf{921.8}             & 28.35   	\\
	BattleZone       	    & 2360     	& 37187.5 	& 22110.0                    & \textbf{26295.0}           & 12.01  		\\
	BeamRider        	    & 363.9    	& 16926.5 	& \textbf{7025.2}            & 6999.3                     & -0.15 		\\
	Berzerk          		& 123.7    	& 2630.4  	& 615.8                      & \textbf{804.6}             & 7.53  		\\
	Bowling          		& 23.1     	& 160.7   	& \textbf{28.8}              & 25.4                       & -2.47 		\\
	Boxing           		& 0.1      	& 12.1    	& \textbf{98.0}              & 96.2                       & -1.83 		\\
	Breakout         		& 1.7      	& 30.5    	& 356.8                      & \textbf{409.4}             & 14.81  		\\
	Carnival         		& -		    & -		    & 5079.4                     & \textbf{5624.9}            &      -      \\
	Centipede        	    & 2090.9   	& 12017.0   & 3088.7                     & \textbf{3416.2}            & 3.29  		\\
	ChopperCommand          & 811.0	    & 7387.8  	& 821.0                      & \textbf{824.0}			  & 0.04  		\\
	CrazyClimber     	    & 10780.5  	& 35829.4 	& \textbf{125451.5}          & 119479.5                   & -5.20 		\\
	DemonAttack             & 152.1    	& 1971.0 	& \textbf{21292.5}           & 12223.0                    & -42.90 		\\
	DoubleDunk       	    & -18.6    	& 16.4    	& {\textbf{-9.6}}    		 & -13.5                      & -11.14 		\\
	ElevatorAction   	    & -		    & -		    & 5197.5                     & \textbf{6763.0}            &   - 		\\
	Enduro           		& 0.0       & 860.5   	& 954.7                      & \textbf{1152.0}            & 20.66  		\\
	FishingDerby     	    & -91.7    	& -38.7   	& \textbf{3.4}               & -0.3                       & -3.89  		\\
	Freeway          		& 0.0      	& 29.6    	& {\textbf{33.6}}    		 & 33.3                       & -0.89		\\
	Frostbite        		& 65.2     	& 4334.7  	& 2823.2                     & \textbf{3396.4}            & 13.42  		\\
	Gopher           		& 257.6    	& 2412.5  	& \textbf{12880.5}           & 8599.4                     & -33.91 		\\
	Gravitar         		& 173.0   	& 3351.4  	& 439.0                      & {\textbf{662.5}}			  & 7.03  		\\
	HERO             		& 1027.0    & 30826.4 	& {\textbf{25319.6}} 		 & 19555.1                    & -19.34 		\\
	IceHockey        	    & -11.2    	& 0.9     	& -10.0                      & \textbf{-7.5}              & 20.66  		\\
	Jamesbond        	    & 29.0   	& 302.8   	& 649.8                      & \textbf{676.3}             & 4.26  		\\
	JourneyEscape           & -		    & -		    & \textbf{-1457.5}           & -1742.5                    &  -        	\\
	Kangaroo         	    & 52.0     	& 3035.0 	& \textbf{3434.5}  			 & 2466.0                     & -28.63 		\\
	Krull            		& 1598.0    & 2665.5  	& 8259.8                     & \textbf{8433.6}            & 2.60    		\\
	KungFuMaster            & 258.5    	& 22736.3 	& 13632.0                    & \textbf{19339.5}           & 25.39  			\\
	MontezumaRevenge        & 0.0       & 4753.3  	& 0.0                        & 0.0                        & 0.00            \\
	MsPacman        	    & 307.3    	& 6951.6  	& 2131.3                     & \textbf{2434.2}            & 4.55  			\\
	NameThisGame 	        & 2292.3   	& 8049.0 	& 9275.0                     & \textbf{10293.3}           & 14.58  			\\
	Phoenix				    & 761.4    	& 7242.6  	& 12035.3                    & \textbf{12719.9}           & 6.07  			\\
	Pitfall          		& -229.4   	& 6463.7  	& \textbf{-1.7}              & -46.1                      & -0.66 			\\
	Pong             		& -20.7    	& 14.6    	& \textbf{20.8}              & 20.7                       & -0.24 			\\
	Pooyan           		& -		    & -			& \textbf{4116.6}            & 2476.5                     &  -        		\\
	PrivateEye       	    & 24.9     	& 69571.3 	& 0.0                        & \textbf{0.5}               & 0.00  			\\
	Qbert            		& 163.9    	& 13455.0   & 14804.6                    & \textbf{14897.9}           & 0.63  			\\
	Riverraid        		& 1338.5   	& 17118.0   & 11045.8                    & \textbf{11422.7}           & 2.38  			\\
	RoadRunner              & 11.5     	& 7845.0 	& \textbf{56092.5}           & 55692.0                    & -0.71 			\\
	Robotank         	    & 2.2      	& 11.9    	& 44.4                       & {\textbf{56.6}}     		  & 28.90  			\\
	Seaquest         	    & 68.4     	& 42054.7 	& 6359.5                     & {\textbf{11055.7}}  	      & 11.18  			\\
	Skiing           		& -17098.1 	& -4336.9 	& -29974.3                   & {\textbf{-11665.5}} 	      & 143.40  		\\
	Solaris          		& 1236.3   	& 12326.7 	& {\textbf{3899.2}}  		 & 3332.1                     & -5.11 			\\
	StarGunner       	    & 664.0   	& 10250.0   & \textbf{64750.0}           & 59603.5                    & -8.03 			\\
	Tennis           		& -23.8    	& -8.3    	& 10.0                       & {\textbf{21.4}}     	      & 33.72  			\\
	TimePilot        		& 3568.0    & 5229.2  	& 3800.0                     & \textbf{6557.5}            & 165.90  		\\
	Tutankham        	    & 11.4     	& 167.6   	& \textbf{144.2}             & 89.2                       & -35.21 			\\
	UpNDown          	    & 533.4    	& 11693.2 	& 12493.2                    & \textbf{17059.2}           & 38.17  			\\
	Venture          		& 0.0       & 1187.5  	& 932.0                      & \textbf{1251.5}   		  & 26.90  			\\
	WizardOfWor             & 563.5    	& 4756.5  	& 2286.0                     & \textbf{4326.5}            & 48.66  			\\
	YarsRevenge             & 3092.9   	& 54576.9 	& 26539.8                    & \textbf{30567.5}           & 7.82   			\\
	Zaxxon           		& 32.5     	& 9173.3  	& 4066.5                     & \textbf{8632.5}            & 49.95			\\
	\bottomrule
\end{tabular}
\end{table}

\clearpage

\section{Experiments on PPER \\ \qquad with A Smaller Fully-connected Layer of \tdpred}
\label{appendix:subsection:PPN with smaller layer}

Our \tdpred architecture is reasonable in the context of following assumptions: 1) it should contain convolutional layers as much as the DQN has, 2) reusing DQN's convolutional layers is important to establish the common ground of DQN and \tdpred (See \Secref{subsection:PPN}; Appendix~\ref{appendix:subsection:PPN design} for details). However, as we introduced in \Secref{sec:Atari experiments}, the number of neurons $M$ in \tdprednos's fully-connected layer is an important hyperparameter of proposed method: it trades off priority and diversity of experience sampling, and also the complexity of algorithm. Here we provide detailed experimental results including test scores.

\vspace{+1em}
\begin{figure*}[h!]
\includegraphics[width=\columnwidth]{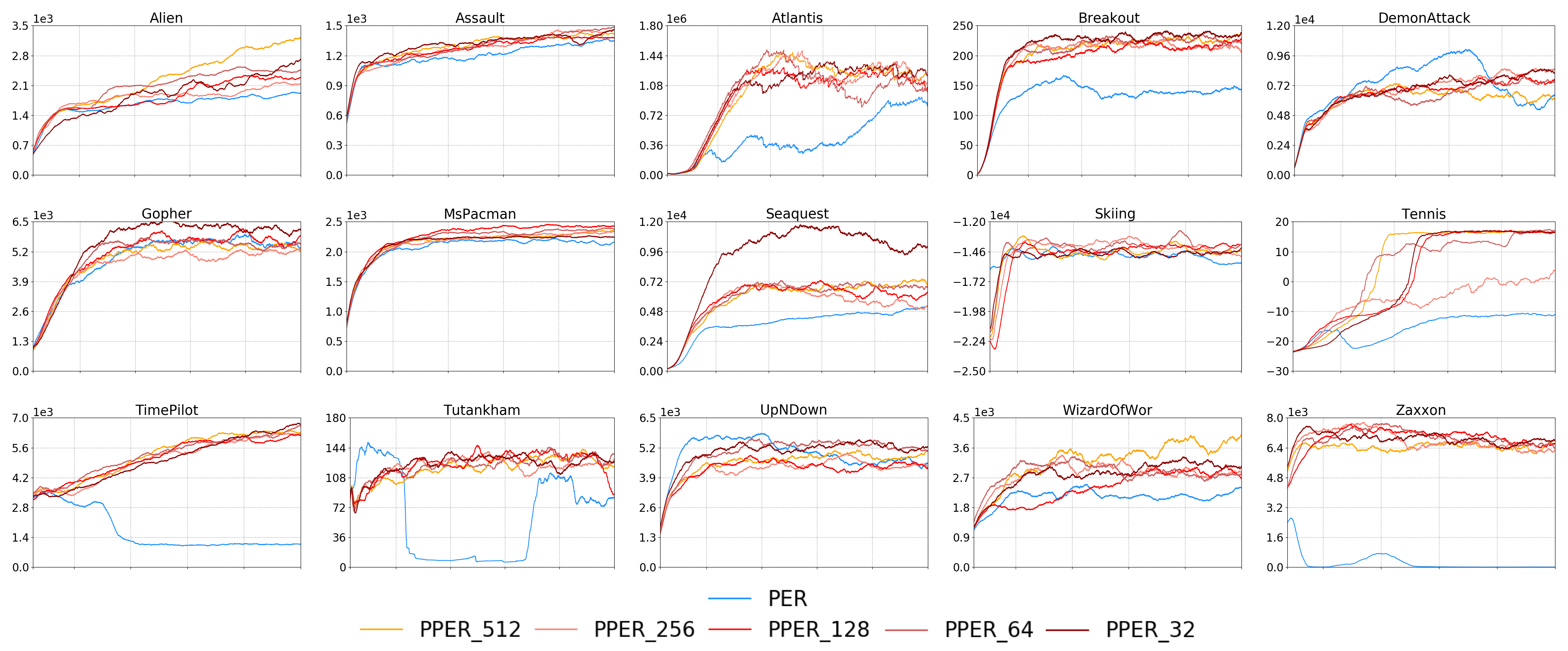}
\vspace{-2em}
\caption{Training curves in raw scores up to 50 million frames of PPER for the number of neurons $M=32, 64, 128, 256, 512$ in \tdprednos's fully connected layer.}	
\label{fig:training curves smaller fc}
\end{figure*}
\vspace{+1em}

\Figref{fig:training curves smaller fc} shows training curves of 15 games, for all variations of PPER. We first observe that the ability of PPER---stable and steady learning---is not sensitive to $M$ in most cases. Also there is no general tendency of relation between $M$ and training performance. This is probably because each environment has different characteristic, requiring different balancing of priority and diversity.

We summarized the no-ops start test results in Table~\ref{tab:scores smaller fc}, where each score is averaged over $200$ repetitions. Based on this experimental result, we hypothesize that the capability of PPN does not significantly affect the effectiveness of the PPER algorithm, possibly the precision of TD-error approximation is less critical in the prioritization scheme, as in rank-based PER \citep{PER2015}.

\begin{table} [b!]
	\footnotesize
	\caption{Test scores of 15 Atari 2600 games.}
	\label{tab:scores smaller fc}
	\centering
	\begin{tabular}{l||r|r|r|r|r|r}
	    \toprule 
	    \\[-1.3em]    
		 \multirow{2}{*}{Games} 
		 & 
		 \multicolumn{6}{c}{Scores}
		 \cr
		 \cline{2-7}
		 & $\;\;$PER$\;\;$ & $\;\;\;$PPER\_512$\;\;\;$ & $\;$PPER\_256$\;$ & PPER\_128 & PPER\_64 & PPER\_32 \\
	    \hdrule
	Alien           & 2220.5            & \textbf{5009.5}   & 2354.0            & 2850.6            & 2786.8            & 2410.0  		        \\
	Assault         & 3656.0            & 2371.8            & \textbf{3995.0}   & 3214.0            & 3993.2            & 2927.4          \\
	Atlantis        & 657339.0          & 699726.0          & 744829.5          & \textbf{763022.5} & 753323.5          & 625574.0  		\\
	Breakout        & 356.8             & 409.4             & 413.6             & 235.8             & 432.9             & \textbf{440.6}  \\
	DemonAttack     & \textbf{21292.5}  & 12223.0           & 16835.4           & 12304.4           & 15513.5           & 16937.4 		\\
	Gopher          & \textbf{12880.5}  & 8599.4            & 6546.8            & 11138.4           & 10587.4           & 12395.6 		\\
	MsPacman        & 2131.3            & 2434.2            & 2306.2            & \textbf{2701.5}   & 2694.2            & 2351.6  	    \\
	Seaquest        & 6359.5            & 11055.7           & 8725.6            & 11495.8           & 13061             & \textbf{16210.5}\\
	Skiing          & -29974.3          & -11665.5          & -26158.3          & \textbf{-9379.1}  & -29973.7          & -29974.8        \\
	Tennis          & 10.0              & 21.4              & 18.0              & 20.7              & 21.8              & \textbf{21.8}   \\
	TimePilot       & 3800.0            & 6557.5            & 6369.5            & 6277.0            & \textbf{7363.0}   & 6414.5          \\
	Tutankham       & 144.2             & 89.2              & \textbf{214.8}    & 180.7             & 117.4             & 75.5 		        \\
	UpNDown         & 12493.2           & \textbf{17059.2}  & 6077.1            & 11480.4           & 7818.9            & 7735.9  		\\
	WizardOfWor     & 2286.0            & \textbf{4326.5}   & 2566.5            & 1513.5            & 3660.0            & 4081.0  		\\
	Zaxxon          & 4066.5            & 8632.5            & \textbf{9166.5}   & 8526.0            & 8677.0            & 8433.0		    \\
	\bottomrule
\end{tabular}
\end{table}

The computational cost of \tdprednos's update dominates the computational overhead of PPER. However, in Table~\ref{tab:computational_cost} we computed and listed computational costs of the \tdprednos's tasks. Updating the \tdpred used in experiments adds up about 10.9\% of computational cost without GPU usage, which is not a significant increase. This is because \tdpred of proposed method is designed to simply reuse the outputs of the convolutional layers of the DQN $Q_\theta$ for each frame input $s$ and $s'$, since both $\delta$ and \smash{$\hat \delta$} are computed at the same place in the batch loop of PPER (lines~\ref{line:PPERBatchUpdateStart}--\ref{line:PPERAccUpdate}). 

\begin{table}[h!]
    \caption{Relative computational costs of \tdprednos's task, 100\% means the cost is equal to DQN's.}
    \begin{center}
    \begin{tabular}{c c c}
        \toprule
         & Computing TD-error ${\hat \delta}_\vartheta$ & Updating the weights $\vartheta$ \\
        \midrule
        PPER\_512 & 66.3\% & 10.9\% \\
        PPER\_256 & 62.5\% &  7.5\% \\
        PPER\_128 & 60.9\% &  6.6\% \\
        \bottomrule
    \end{tabular}
    \end{center}
    \label{tab:computational_cost}
\end{table}

Another notable fact is that the \tdpred can predict TD-errors with approximately 33\% less computation: it only requires a single feed-forwarding, while computing the raw TD error for the double DQN bootstrapped target requires feed-forwarding three times. 

\begin{figure*}[b!]
\subfigure[Priority distributions of Alien]{\includegraphics[width=\columnwidth]{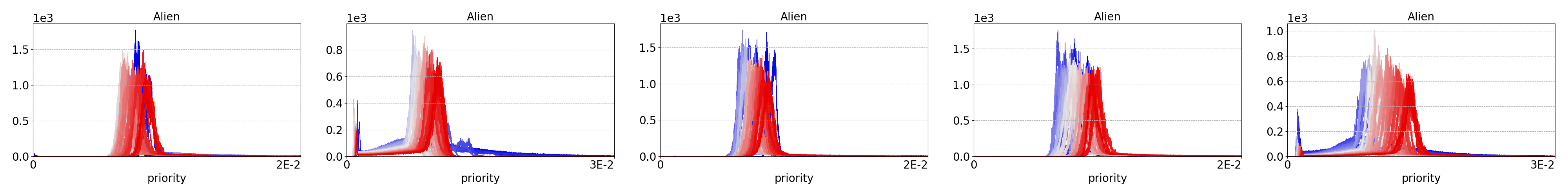}}
\subfigure[Priority distributions of Atlantis]{\includegraphics[width=\columnwidth]{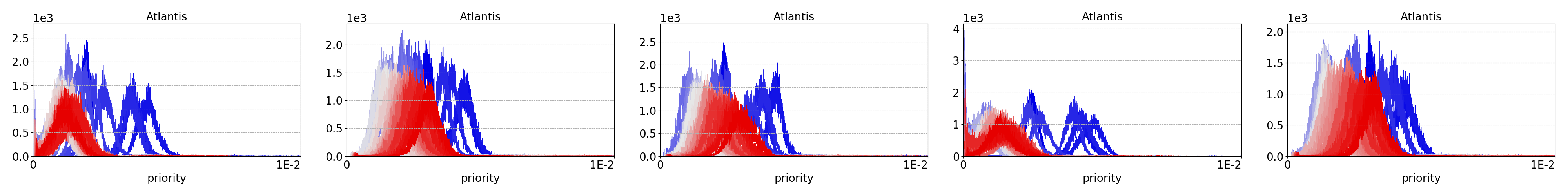}}
\subfigure[Priority distributions Seaquest]{\includegraphics[width=\columnwidth]{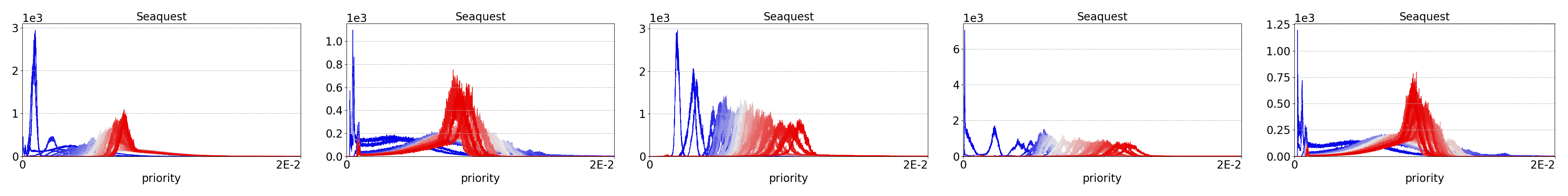}}
\subfigure[Priority distributions TimePilot]{\includegraphics[width=\columnwidth]{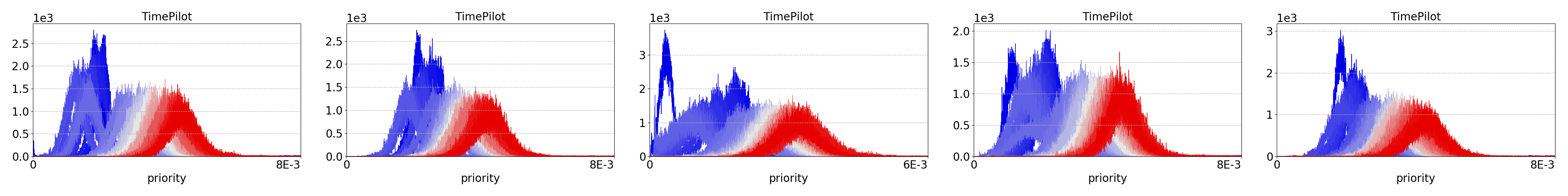}}
\subfigure[Priority distributions UpNDown]{\includegraphics[width=\columnwidth]{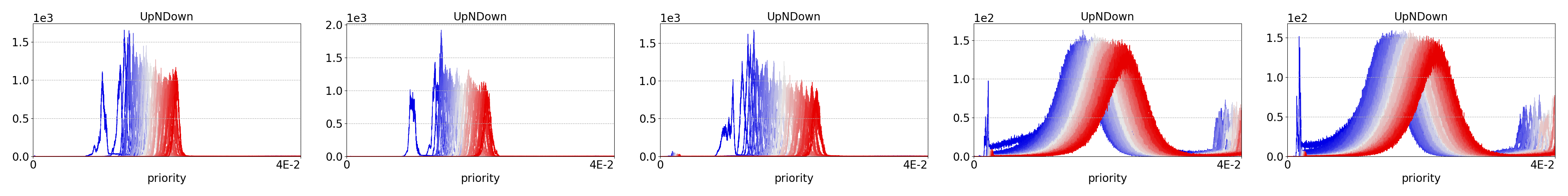}}
\caption{Priority distribution changes over the training for all variations (From left: PPER\_32, PPER\_64, PPER\_128, PPER\_256 and PPER\_512). The colors of distributions change from {\color{blue}blue} to {\color{red}red}, as the training steps increase up to $T_\mathsf{max} = 50$M.}
\label{fig:appendix:priority distribution smaller fc}
\end{figure*}

\clearpage
\setcounter{equation}{0}

\section{Statistics of TD-error $\delta$ and \tdpred Output $\hat \delta$}
\label{appendix:variance reduction}

\begin{figure*}[h!]
\includegraphics[width=\columnwidth, height=20cm]{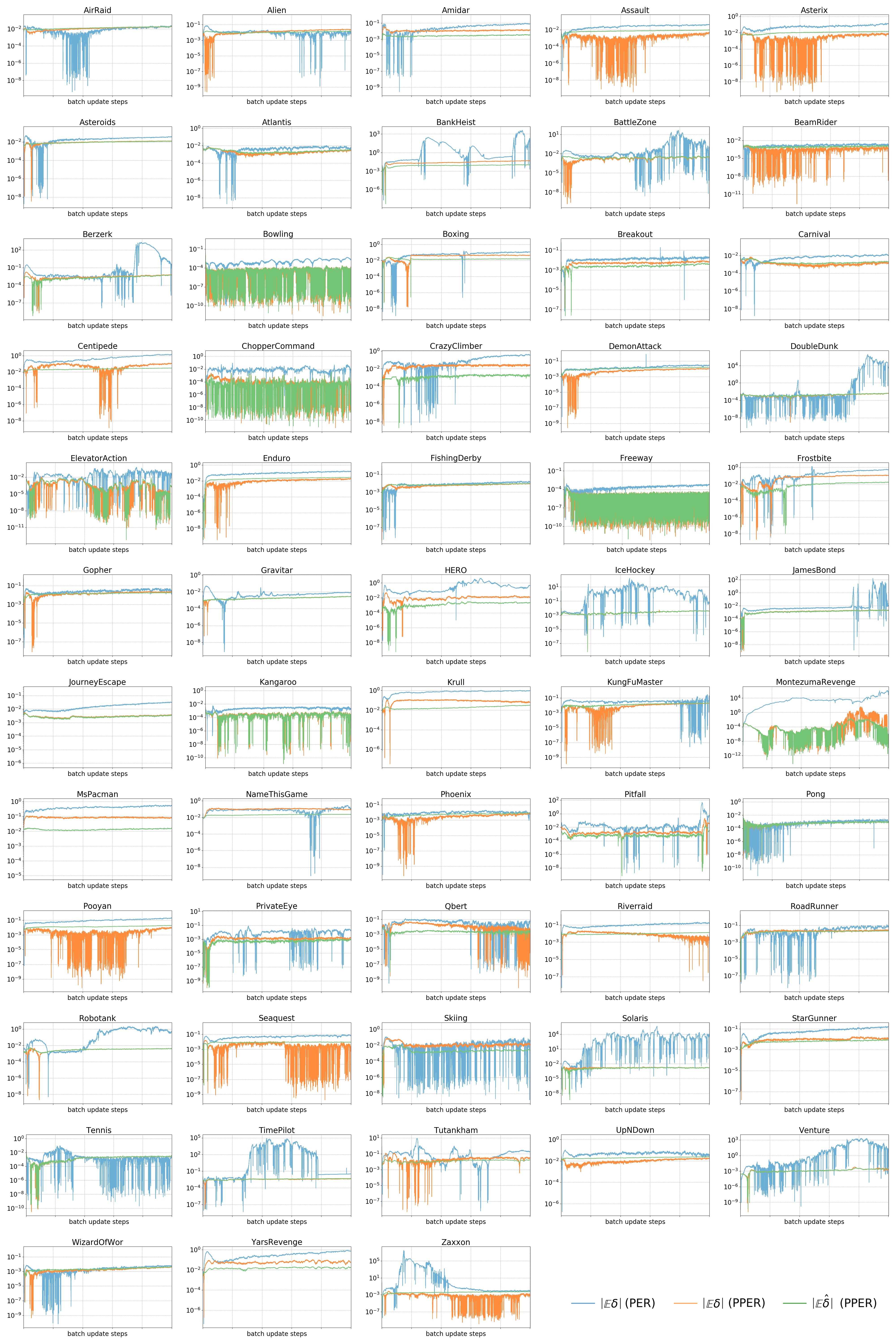}
\caption{Magnitudes of averages of batch mean $\delta$ (PER, PPER) and ${\hat \delta}$ (PPER), over 58 Atari game. Note that (i) the mean of $\hat \delta$ (PPER, {\color{green}green}) follows that of $\delta$ (PPER, {\color{orange}orange}) quite accurately, with small biases; (ii) PER ({\color{blue}blue}) shows huge priority explosions for some cases, while PPER does not.}
\vspace{-1em}
\end{figure*}

\begin{figure*}[h!]
\includegraphics[width=\columnwidth, height=20cm]{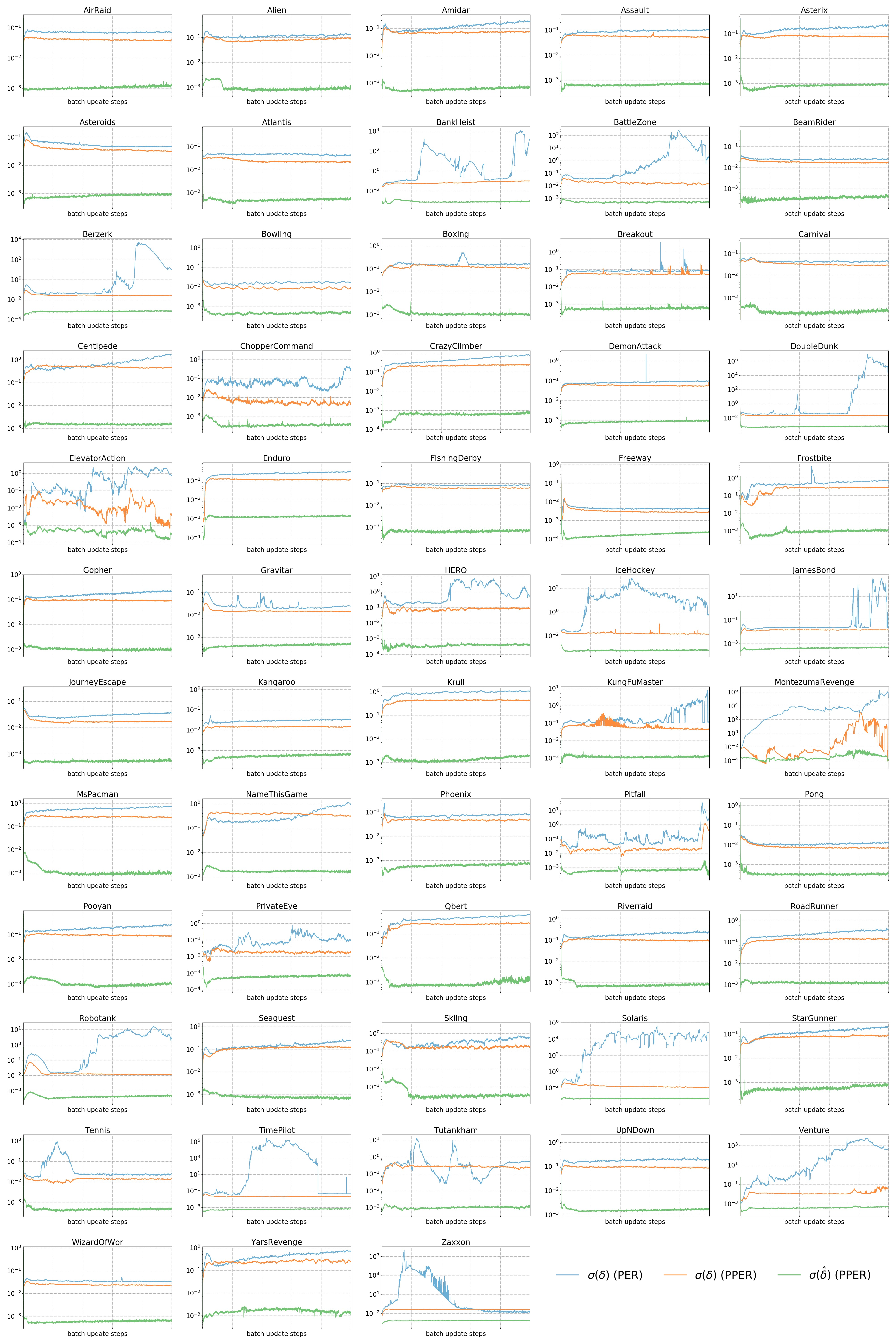}
\caption{Standard deviations of batch mean $\delta$ (PER, PPER) and ${\hat \delta}$ (PPER) over 58 Atari games. Note (i) the reduced standard deviation of $\hat \delta$ (PPER, {\color{green}green}) and that (ii) the standard deviation of $\delta$ in PPER ({\color{orange}orange}) is more stabilized than that in PER ({\color{blue}blue}).}	
\vspace{-1em}
\end{figure*}

\end{document}